%% file: main.tex
\documentclass{article}

\usepackage{arxiv}

\usepackage[utf8]{inputenc} 
\usepackage[T1]{fontenc}    
\usepackage{hyperref}       
\usepackage{url}            
\usepackage{booktabs}       
\usepackage{amsmath}
\usepackage{amsfonts}       
\usepackage{nicefrac}       
\usepackage{microtype}      
\usepackage[capitalize]{cleveref}       
\usepackage{graphicx}
\usepackage{natbib}
\usepackage{bm}
\usepackage{mathabx}
\usepackage{caption}
\usepackage{subcaption}
\usepackage{tabularx}
\usepackage{multirow}
\usepackage{multicol}
\usepackage{array}    
\usepackage{wrapfig}

\setcitestyle{authoryear,round,citesep={;},aysep={,},yysep={;}}

\input{math_commands.tex}

\usepackage{amsthm}
\theoremstyle{plain}
\newtheorem{theorem}{Theorem}[section]

\newtheorem{lemma}[theorem]{Lemma}

\theoremstyle{definition}
\newtheorem{definition}[theorem]{Definition}

\theoremstyle{remark}
\newtheorem{remark}[theorem]{Remark}
\title{Norm-Bounded Low-Rank Adaptation}

\title{Norm-Bounded Low-Rank Adaptation}

\author{%
  Ruigang Wang$^{1}$, Krishnamurthy (Dj) Dvijotham$^2$, Ian R. Manchester$^1$ \\
  $^1$Australian Centre for Robotics, School of Aerospace, Mechanical and \\
  Mechatronic Engineering, The University of Sydney\\
  $^2$Google DeepMind\\
  \texttt{\footnotesize \{ruigang.wang, ian.manchester\}@sydney.edu.au}, \\ \texttt{\footnotesize  dvij@cs.washington.edu}\\
}

\renewcommand{\shorttitle}{Norm-Bounded Low-Rank Adaptation}

\begin{document}
\maketitle

\begin{abstract}
In this work, we propose norm-bounded low-rank adaptation (NB-LoRA) for parameter-efficient fine tuning. NB-LoRA is a novel parameterization of low-rank weight adaptations that admits explicit bounds on each singular value of the adaptation matrix, which can thereby satisfy any prescribed unitarily invariant norm bound, including the Schatten norms (e.g., nuclear, Frobenius, spectral norm). The proposed parameterization is  unconstrained, smooth, and complete, i.e. it covers all matrices satisfying the prescribed rank and singular-value bounds. Natural language generation experiments show that NB-LoRA matches or surpasses performance of competing LoRA methods, while exhibiting stronger hyper-parameter robustness.  Vision fine-tuning experiments show that NB-LoRA can avoid model catastrophic forgetting without minor cost on adaptation performance, and compared to existing approaches it is substantially more robust to a hyper-parameters such as including adaptation rank, learning rate and number of training epochs.
\end{abstract}


\section{Introduction}
Large pretrained vision and language models have demonstrated impressive generalization capability across a wide variety of tasks; see, e.g. \citet{achiam2023gpt,touvron2023llama}. When a more specific target task is identified, however, it has been observed that parameter-efficient fine-tuning (PEFT) techniques, e.g.  \citet{houlsby2019parameter,hu2022lora}, can improve performance via quick model adaption with low computation and data requirements. The primary goal for an effective PEFT method is to achieve good adaptation performance with high training efficiency, i.e., dramatically fewer trainable parameters and training epochs. Since training efficiency is the target, ideally such a method will be quite robust to hyper-parameters. Alongside this primary goal, it is often also desirable to maintain the generalization performance of the original pretrained model as much as possible, i.e. avoid ``catastrophic forgetting'' \citet{qiu2023controlling,biderman2024lora}. 

Low-rank adaption (LoRA) \citep{hu2022lora} is a widely applied PEFT method, which parameterizes the update of pretrained weights $W_p\in \R^{m\times n}$ during finetuning as 
\begin{equation}\label{eq:lora}
    y=(W_{p}+W)x=\left(W_{p}+\frac{\alpha}{r}B^\top A\right)x
\end{equation}
where $A\in \R^{r\times n}, B\in \R^{r\times m}$ are the learnable matrices, $\alpha$ is a scaling factor, and $r\ll \min(m,n)$ is the rank budget of weight adaptation $W$. Matrix rank is one way to quantify the ``size'' of a weight, corresponding the underlying dimensionality of its operation. But matrix norms -- such as nuclear, Frobenius, or spectral norms -- provide another notion of size, quantifying the magnitude of a matrix's elements and of its operation on vectors. 

Recent works show that it is beneficial to control the rank and norm of the weight adaption.  \citet{jang2024lora,kim2025lora} show that the global minimum of fine-tuning has low rank and small magnitude while spurious local minima (if they exist) have high rank and large magnitude. Moreover, bounding the magnitude of $W$ can enhance training robustness \citep{bini2025decoupling}. In \citet{hu2025computational}, LoRA training can achieve sub-quadratic time complexity under certain norm-bound conditions.

Motivated by those findings, we propose norm-bounded low-rank adaptation (NB-LoRA), a novel finetuning method that admits explicit bounds on both the rank \emph{and} norm of weight update through matrix reparameterization (see Fig.~\ref{fig:nblora-scheme}). Our approach can control a family of matrix norms, called Schatten $p$-norms (i.e. $p$-norms of the singular value sequence), which include the nuclear norm, Frobenius norm, and spectral norm as special cases. We summarize our contributions as follows. 

\begin{figure}[!t]
    \centering
    \includegraphics[width=0.9\linewidth]{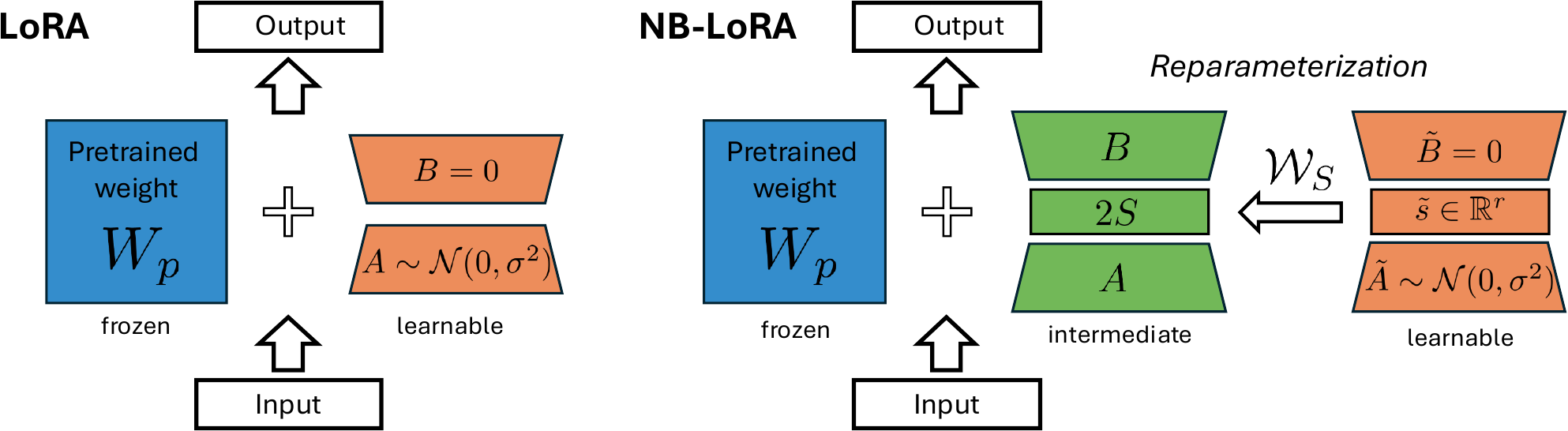}
    \caption{Visualization (Left) of the original LoRA \cite{hu2022lora} and (Right) of our proposed method NB-LoRA, where bounded rank and norm are enforced by reparameterization $\mathcal{W}_S$.} 
    \label{fig:nblora-scheme}
\end{figure}

\begin{itemize}
    \item Our parameterization is a smooth map $W=\mathcal{W}_S(\tilde A, \tilde B)$ which takes as argument two free matrix variables of the same size as $A, B$, but the resulting $W$ automatically satisfies user-prescribed bounds on both rank and all individual singular values of $W$, which further allows any Schatten $p$-norm bound on $W$ to be specified. 
    \item Our parameterization is {\em complete}, i.e., for any $W\in \R^{m\times n}$ satisfying the prescribed bounds on singular values, there exists a (not necessarily unique) $\tilde A, \tilde B$ such that $W=\mathcal W_S(\tilde A, \tilde B)$.
    \item LLM fine-tuning experiments  show that NB-LoRA can substantially improve training stability, overall performance and robustness to learning rates.
    \item Vision transformer fine-tuning experiments illustrate that NB-LoRA can achieve similar adaptation performance to LoRA and other existing methods while exhibiting less ``forgetting'' of the source model. Also, norm bounds appear to significantly reduce sensitivity to hyperparameter variation.
\end{itemize}

\section{Related Work}

LoRA can be highly sensitive to learning rate \citep{bini2024ether, biderman2024lora}, model initialization \citep{hayou2024impact}, and it is susceptible to over-training \citep{qiu2023controlling}.  To mitigate these effects, several recent works have proposed regularization techniques for LoRA. For example, \citet{gouk2021distance, chen2023parameter} propose an approach that preserves the Euclidean weight distances between pre-trained and fine-tuned models. In \citet{liu2024dora}, DoRA was proposed based on investigation of the vector-wise norm of the adaption matrix, and introduces an adaptive scaling of $W$. \citet{bini2025decoupling} proposed DeLoRA - a PEFT method that decouples the angular learning from adaptation strength. VeRA is another method which learns a scaling vector for LoRA weights \citep{kopiczko2024vera}. Our method also contains a learnable scaling vector, which can be used to explicitly control bounds on each singular value of the weight adaptation.

Another line of LoRA methods are closely related to singular value decomposition (SVD).  \citet{meng2024pissa} proposed a novel SVD-based LoRA initialization, called PiSSA, which can significantly speed up the training of LoRA. \citet{zhang2023adaptive} proposed a dynamical rank allocation scheme, called AdaLoRA, which adaptively update the rank bound in each LoRA layer. In \citet{lingam2024svft,balazy2024lora}, the singular vectors of pretrained weights are re-used and a small square matrices are learned during fine-tuning. No explicit control of norm bounds or constraint on singular values were considered in these methods.  

\section{Motivating Analysis of LoRA}\label{sec:lora}
In this section we provide some brief analyses of LoRA that motivate consideration of alternative parameterizations of $A$ and $B$.

\paragraph{Analysis of Gradients and Initialization.} For LoRA the standard approach is to initialize with one of $A$ or $B$ equal to zero, so that $W=0$ initially, and the other as a small random matrix to enable learning but avoid training instability (see \citet{hayou2024impact} for a discussion of approaches). Based on \cref{eq:lora}, we can obtain the gradients of the loss $\ell(\cdot)$ with respect to $A$ and $B$ as follows:
\begin{equation*}
    \frac{\partial \ell}{\partial A}=\frac{\alpha}{r}B \left(\frac{\partial \ell}{\partial y}\right)x^\top,\quad \frac{\partial \ell}{\partial B}=\frac{\alpha}{r}Ax\left(\frac{\partial \ell}{\partial y}\right)^\top.
\end{equation*}
The component $\partial \ell/\partial y$ is typically not large, since the pre-trained base model has reasonable generalization capability over a wide range of tasks. Thus, at the beginning of fine-tuning, if $B=0$ and $A$ is a small random matrix, then the gradient $\partial\ell/\partial A=0$ and $\partial\ell/\partial B$ is small and can be noisy. Therefore, both gradients could be small and uninformative for a large number of training steps, leading to slow convergence and poor performance since fine-tuning is often carried out for a few epochs. Large learning rate can help to speed up but it may cause training instability. The above phenomenon has been reported and analyzed in a recent work \citep{meng2024pissa}, which proposed PiSSA as an alternative initialization.

In this work, we provide a novel model reparameterization, and we note that this impacts training behavior since gradient descent is affected by changes of coordinates. In our approach, the low-rank matrices $A,B$ are constructed from new free variables $\tilde{A}, \tilde{B}$ which share the same size as $A$ and $B$, respectively. Under this reparameterization (detailed in Section \ref{sec:theory}), $A$ and $B$ cannot be both very small matrices. For example, if $B$ is a zero matrix, then by construction $A$ is a relatively large matrix, which in turn provides sufficiently large gradient for $A$ to move away from zero in a few steps. However, the Frobenius norms of $
A,B$ are guaranteed to be bounded, thus the gradients of $A,B$ are also bounded assuming that $x\left(\partial \ell/\partial y\right)^\top$ is bounded, and thus large gradient steps can be taken without training instability, i.e. we expect norm bounds to assist with robustness to learning rate. This informal argument is supported by experimental results in Section \ref{sec:llm}.

\paragraph{Analysis of Model Forgetting.} Let $\{(x_i^s, y_i^s)\}_{1\leq i\leq M}$ be the pretrained input-output pair of \eqref{eq:lora} under the source training dataset $\mathcal{D}_S$. After fine-tuning the adaption weight $W$ based on some target dataset $\mathcal{D}_T$, we can approximate the loss changes on $\mathcal{D}_S$ by 
\[
\ell_{\mathcal{D}_s}(W_p+W)-\ell_{\mathcal{D}_s}(W_p)\approx \frac{1}{M}\sum_{i=1}^{M}  \frac{\partial \ell}{\partial y_i^s} W x_i^s.  
\]
To prevent catastrophic forgetting we need to bound the left-hand side. 
Since $x_s$ and $\frac{\partial \ell}{\partial y_i^s}$ are fixed, we argue that constraining the norm of $W$ is a natural approach. 
If we have access to $\mathcal{D}_S$ during finetuning, one could incorporate source loss into the training to mitigate forgetting. However, $\mathcal{D}_S$ is often not available in fine-tuning applications.

\section{NB-LoRA}\label{sec:theory}
In this section we present our main contribution: a parameterization of low-rank matrices that admits bounds on each individual singular value, and hence on any unitarily invariant matrix norm.

\subsection{Preliminaries and Problem Formulation}

The problem we are interested in can be formalized as follows: 
\begin{equation} \label{eq:problem}
    \min \quad \ell(W)\quad \mathrm{s.t.}\quad \rank(W)\leq r,\; \|W\|_{S_p} \leq \delta 
\end{equation}
with $\ell$ as the training loss and $\|W\|_{S_p}=\bigl(\sum_{i=1}^r \sigma_i^p\bigr)^{1/p}$ for $p\in [1,\infty)$ and $\|W\|_{S_\infty}=\sigma_1$, where $\sigma_1\geq\sigma_2\geq  \cdots \geq \sigma_r\geq 0$ are the singular values of $W$. Since Schatten $p$-norm is the vector $p$-norm of the singular value sequence, it is unitarily invariant, i.e., $\|W\|_{S_p}=\|UWV\|_{S_p}$ for any orthogonal matrices $U,V$. 

We first define some notation. Since our approach involves comparing singular values of matrices of potentially different ranks and sizes, for convenience we define $\sigma_j(W)=0$ if $j>\mathrm{rank}(W)$. We now introduce the  relation $\preceq_\sigma$. 
\begin{definition}
    Let $X,Y$ be two matrices. We say $X \preceq_\sigma Y $ if $\sigma_j(X)\leq \sigma_j(Y),\, \forall j\in \mathbb{N}$. 
\end{definition}
Note the $\preceq_\sigma$ is reflexive ($X\preceq_\sigma X$) and transitive ($X\preceq_{\sigma} Y,\, Y\preceq_{\sigma} Z\Rightarrow X\preceq_{\sigma} Z$). But it is not  antisymmetric, i.e., $X\preceq_{\sigma} Y,\, Y\preceq_{\sigma} X\nRightarrow X= Y$, e.g., when $X,Y$ are distinct orthogonal matrices. Most importantly for our purposes: if $X\preceq_{\sigma} Y$, then $\|X\|_{S_p}\leq \|Y\|_{S_p}$ for all $p\in [1,\infty]$.

Let $s\in\R_{+}^{r}$, where $\R_{+}=[0,\infty)$, and  $S=\diag(s)$ be the diagonal matrix with $S_{jj}=s_j$. We define the set of matrices whose singular values are bounded by $S$ by
\begin{equation*}
    \sW_S:=\{ W\in \R^{m\times n}\mid  W\preceq_\sigma S\}.
\end{equation*}
Note that for any $W\in \sW_S$, we have $\rank(W)\leq \rank(S)=r$ and $\|W\|_{S_p}\leq \|S\|_{S_p}$.

\subsection{NB-LoRA Parameterization}
We now present so-called \textit{direct} parameterization of $\sW_S$, a smooth mapping $\mathcal{W}_S$ from free matrix variables to $W$ which maps onto the entire set $\sW_S$. Then, we can transform (\ref{eq:problem}) into an unconstrained problem by further parameterizing the positive diagonal matrix $S$ such that $\|S\|_{S_p}=\delta$. 

Our parameterization takes $\tilde{A}\in \R^{r\times n}, \tilde{B}\in \R^{r\times m}$ as the free parameters and produces $W$ via
\begin{equation}\label{eq:w-param}
    W=\mathcal{W}_S(\tilde A, \tilde B):=2B^\top S A \textrm{, \   where} \  \begin{bmatrix}
        A^\top \\ B^\top 
    \end{bmatrix}=\cayley\left(\begin{bmatrix}
        \tilde{A}^\top \\ \tilde{B}^\top
    \end{bmatrix}\right).
\end{equation}
Here the Cayley transformation for a tall matrix $\begin{bmatrix}
        X \\ Y
    \end{bmatrix}$ with $X\in \R^{r\times r}$ and $Y\in \R^{q\times r}$ is defined by
\begin{equation}\label{eq:cayley}
    \cayley\left(\begin{bmatrix}
        X \\ Y
    \end{bmatrix}\right):=\begin{bmatrix}
        (I-Z)(I+Z)^{-1} \\
        -2Y(I+Z)^{-1}
    \end{bmatrix}, \;\text{where } Z=X-X^\top+Y^\top Y.
\end{equation}
Note that $G=\cayley(F)$ is a semi-orthogonal matrix, i.e., $G^\top G=I$ for any tall matrix $F$ \citep{trockman2021orthogonalizing}, however it is not by itself a complete parameterization for the set of semi-orthogonal matrices, e.g., there does not exist an $F$ such that $\cayley(F)=-I$. Despite this, we have the following, which is the main theoretical result of the paper.

\begin{theorem}\label{thm:param}
    The NB-LoRA parameterization in (\ref{eq:w-param}) is a direct (smooth and complete) parameterization of $\sW_S$, i.e. $\mathcal{W}_S$ is differentiable and $\mathcal{W}_S(\R^{N})=\sW_S$.
\end{theorem}

\begin{remark}
A special case of the above theorem is $S=I$, which is a complete parameterization of all 1-Lipschitz linear layer, i.e. $f(x)=Wx$ with $\|W\|_{S_\infty}\leq 1$, see Proposition 3.3 of \citet{wang2023direct}.  One can further extend it to a nonlinear layer with low-rank and norm-bounded Jacobian. Specifically, we take a nonlinear layer of the form $f(x)=2B^\top D_1 \phi (D_2Ax)$ where $A,B$ are constructed from (\ref{eq:w-param}), $D_1, D_2$ are diagonal matrices satisfying $0\preceq D_1D_2\preceq S$ and $\phi$ is a scalar activation with slope-restricted in $[0,1]$. Then, we have $\partial f/\partial x\in \sW_S$ for all $x\in \R^n$. 
\end{remark}

\paragraph{Imposing the Norm Bound on $W$.} From \cref{thm:param}, if we construct a complete parameterization for the set of singular bound vector $s\in \R_{+}^r$ such that $\|s\|_p=\delta$, then the proposed NB-LoRA (\ref{eq:w-param}) covers all adaptation matrices $W$ satisfying the prescribed rank and norm bounds. For $p=\infty$, we simply take $s=(\delta, \delta, \ldots, \delta )$.  For $p\in [1,\infty)$, one approach is $s=\delta |\tilde{s}|/\|\tilde{s}\|_p$, where $\tilde{s}\in \R^r$ is a free non-zero vector. However, this parameterization is not smooth at $\tilde{s}=0$. Instead, we use the following parameterization in our experiments:
\begin{equation*}
    s=\delta  \left[\mathrm{Softmax}\left(\tilde{s}/\sqrt{r}\right)\right]^{1/p}.
\end{equation*}
Technically, the above parameterization omits some boundary cases with $\|W\|_{S_p}=\delta$ and $\sigma_r(W)=0$ since softmax has strictly positive outputs. However, since it covers the interior of the feasible set and can approximate the boundary, there is no practical impact on optimization performance.

\paragraph{Model Initialization and Gradient analysis.} Here we return to the motivating analysis from Section \ref{sec:lora} and show why NB-LoRA helps resolve the issue of small gradients. We can adapt the standard LoRA initialization to NB-LoRA's free parameters: sampling $\tilde{A}$ as a small random matrix and setting $\tilde{B}=0$. After applying the Cayley transformation, we have $AA^\top=I$ and $B=0$, yielding a zero initialization for  $W$. The gradients  of $A,B$ can be written as
\begin{equation*}
    \frac{\partial \ell}{\partial A}=2SB \left(\frac{\partial \ell}{\partial y}\right) x^\top=2\hat{B}\left(\frac{\partial \ell}{\partial y}\right)x^\top,\quad \frac{\partial \ell}{\partial B}=2SAx\left(\frac{\partial \ell}{\partial y}\right)^\top=2\hat{A}x\left(\frac{\partial \ell}{\partial y}\right)^\top
\end{equation*}
where $\hat{A},\hat{B}$ satisfy 
\begin{equation*}
    \hat{A}\hat{A}^\top +\hat{B}\hat{B}^\top =S(AA^\top+BB^\top)S= S^2
\end{equation*}
with $\|S\|_{S_p}=\delta>0$. Hence $\hat{A},\hat{B}$ cannot be both arbitrarily small matrices, implying that $\partial \ell/\partial A$ and $\partial \ell/\partial B$ cannot be both arbitrarily small initially. On the other hand, the Frobenius norms of $\hat{A},\hat{B}$ are also bounded by $c\delta$ where $c$ is the constant satisfying $\|S\|_{S_2}\leq c\|S\|_{S_p}$ for all $S$. Thus, if $x(\partial \ell/\partial y)^\top$ remains bounded, then $\partial\ell /\partial A$ and $\partial \ell/\partial B$ are bounded, allowing stable training for a wider range of learning rates than LoRA.

\paragraph{Computational cost of Cayley Transformation.} Due to the low-rank nature ($r$ is often less than 256), computing the inverse of an $r\times r$ matrix is not overly expensive. While matrix inversion is one part of the total training cost, another computationally intensive part is the backward pass for the Cayley transformation (\ref{eq:cayley}). We provide an efficient custom backward step in \cref{sec:cayley-backward}.

\paragraph{DeLoRA vs NB-LoRA.} Similarly to our method, DeLoRA \citep{bini2025decoupling} can also control the Frobenius norm bound of weight adaption based on the following parameterization:   
\begin{equation}\label{eq:delora-residual}
    W=\frac{\gamma}{r} B^\top \Xi A-\frac{\gamma_0}{r} B_0^\top \Xi_0 A_0
\end{equation}
where the scaling factor $\gamma$ and weight parameter $A, B$ are initialized as $\gamma_0$ and $A_0, B_0$. $\Xi$ is a diagonal matrix that normalizes each row of $A, B$. Similar to PiSSA \cite{meng2024pissa}, the second term in (\ref{eq:delora-residual}) can be absorbed into the pretrained weight $W_p$. Note that our method can control other Schatten norms (e.g. nuclear, spectral), which were not considered in \cite{bini2025decoupling}.

\begin{figure}[!t]
    \centering
    \includegraphics[width=0.8\linewidth]{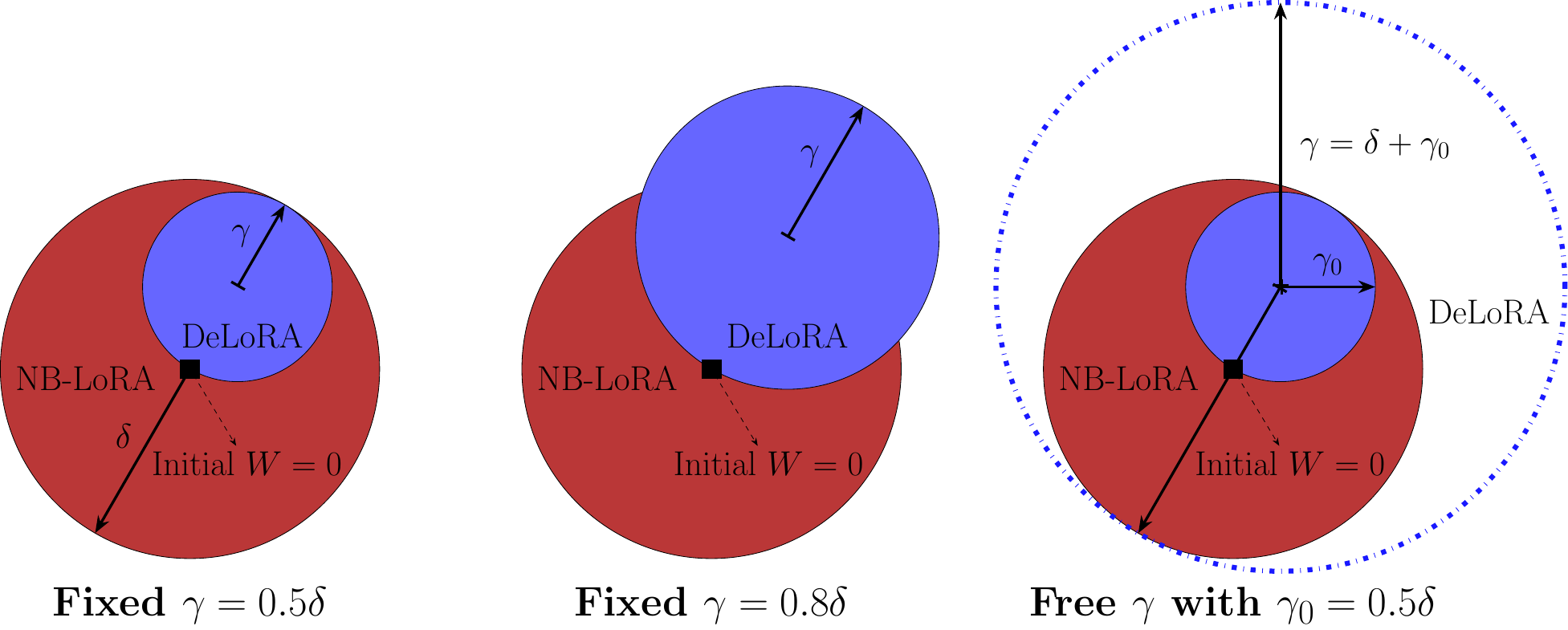}
    \caption{Visualization of the reachable sets $\sW_{\tt NBLoRA}$ (red) and $\sW_{\tt DeLoRA}$ (blue). (Left) With frozen scaling factor $\gamma=0.5\delta$, DeLoRA provides the same certified norm bound as NB-LoRA while $\sW_{\tt DeLoRA}$ is much smaller than $\sW_{\tt NBLoRA}$. (Middle) Further increasing the fixed $\gamma$ can enlarge $\sW_{\tt DeLoRA}$ but it does not cover $\sW_{\tt NBLoRA}$. (Right) DeLoRA can cover $\sW_{\tt NBLoRA}$ if $\gamma$ is free and sufficiently large, i.e., $\gamma\geq \delta+\gamma_0$. However, its norm bound $\delta+2\gamma_0$ is much larger than NB-LoRA.} 
    \label{fig:nblora-delora}
\end{figure}

As shown in \cref{sec:comparison}, both DeLoRA and NB-LoRA can be represented as sum of NB-LoRA matrices with both rank and norm bound of 1. The main difference comes from their reachable sets $ \sW_{\tt NBLoRA}$ and $ \sW_{\tt DeLoRA}$ when an explicit norm bound is specified. From \cref{thm:param} we have that $ \sW_{\tt NBLoRA}$ covers the feasible region of $W$ with norm bound of $\delta$. Moreover, the initial point $W=0$ of NB-LoRA lies at the center of the feasible region, allowing searching for all directions, see \cref{fig:nblora-delora}. Due to the residual type initialization, DeLoRA requires a fixed $\gamma=0.5\delta$ to ensure the same norm bound guarantee, see the left of \cref{fig:nblora-delora}. Since its initial $W$ lies at the boundary of $ \sW_{\tt DeLoRA}$, the searching directions of DeLoRA are constrained in certain ranges that depend on the random initial guess. Although these issues can be resolved by making $\gamma$ learnable, DeLoRA allows an unbounded Frobenius norm, see the right of \cref{fig:nblora-delora}.  

\paragraph{PiSSA vs NB-LoRA.} PiSSA \citep{meng2024pissa} addresses the small initial gradient issue of LoRA via a residual-type initialization, i.e., $W=\frac{\alpha}{r}(B^\top A - B_0^\top A_0)$ where $A_0,B_0$ are the initial values of $A, B$, respectively. Since the term $\frac{\alpha}{r}B_0^\top A_0$ can be absorbed into the pretrained weight $W_p$, it does not cause any extra computation cost compared with LoRA. Moreover, PiSSA has much lager initial gradients than LoRA by constructing $A_0,B_0$ from the reduced SVD of $W_p$. Thus, PiSSA can speed up the fine-tuning process, however, its performance might be sensitive to learning rate as $A, B$ are unbounded. Different from PiSSA, the proposed NB-LoRA approach address the small initial gradient issue through reparameterization. Since $A,B$ live on a compact manifold by construction,  thus NB-LoRA allows for a wide range choice of learning rates.

\paragraph{DoRA vs NB-LoRA.} DoRA \citep{liu2024dora} decouples angular and magnitude components of weight adaptation via $ W=\underline{m} \frac{(W_p+B^\top A)}{\|W_p+B^\top A\|_c}$ with $\|\cdot\|_C$ as the column-wise vector norm. Note that the normalization vector $\|W_p+B^\top A\|_c$ requires computing $B^\top A\in \R^{m\times n}$, whose forward computation time could be much larger than $r\times r$-matrix inverse, see \cref{tab:llm-BA-inv} of \cref{sec:llm-appendix}.  

\section{LLM Fine-Tuning Experiments}\label{sec:llm}

In this section, we evaluate our proposed NB-LoRA method for natural language generation (NLG) tasks. Our main objectives are as follows: i) NB-LoRA can avoid small initial gradients while still maintain training stability for a wide range of learning rates; ii) Controlling the norm is beneficial for robust performance; iii) Due to the ability of tight bound control, our method can outperform existing approaches with the same certified norm bound.

\paragraph{NLG Tasks.} We start with comparing LoRA, DoRA and PiSSA on NLG tasks. We fine-tuned the LLaMA model family \citep{touvron2023llama} and Mistral-7B-v0.1 \citep{jiang2023mistral7b} on the MetaMathQA dataset \citep{yu2023metamath} to evaluate their mathematical problem-solving capability on the GSM8K~\cite{cobbe2021gsm8k} and MATH~\cite{hendrycks2021measuring} test datasets. We also fine-tuned the models on the the CodeFeedback dataset~\cite{zheng2024opencodeinterpreter} and evaluated for coding proficiency using the HumanEval~\cite{chen2021evaluating} and MBPP~\cite{austin2021program}. We adopt the implementation strategy from \cite{alpaca}. We follow the training setup in \cite{meng2024pissa} with default rank budget of $r=128$ and scaling parameter $\alpha=r$ for LoRa and PiSSA. For the proposed NB-LoRA method, since the scaling components in $S$ of (\ref{eq:w-param}) are initialized close to $1/r$, we then choose the nuclear norm bound of $\delta=r$, leading to the same scaling factor as LoRA and PiSSA. More training details can be found in \cref{sec:llm-appendix}.

\paragraph{Large Initial Gradients and Training Stability.} The analysis in  \cref{sec:lora} suggested that norm bounds may improve robustness to learning rates while allowing for large initial gradients. We conducted experiments on LLaMA-2-7B fine-tuning across a wide range of learning rates from 5e-5 to 1e-3 on math and python coding datasets. \cref{fig:llm-lr} (a) and (b) show that LoRA and DoRA both suffer from poor performance with small learning rates, due to the small gradients problem. Increasing the learning rate helps up to a point but then training goes unstable. In contrast, NB-LoRA achieves good performance for a wide range of learning rates, outperforming all other models on most tasks. PiSSA outperforms NB-LoRA in terms of peak performance on GSM8K, but under-performs on other tasks and is more sensitive to learning rate.

\cref{fig:llm-lr} (c) and (d) show loss and gradient norm vs training steps. It can be seen that with a small learning rate (5e-5) NB-LoRA (and PiSSA) train similarly, both faster than LoRA and DoRA and with larger gradient norms. With a larger learning rate (1e-3) LoRA and DoRA were unstable, and NB-LoRA trains fastest. Note that PiSSA has a larger gradient norm but slower training: since the parameterizations are different the gradient norms are not directly comparable.

\paragraph{Hyperparameter Robustness.} \cref{{tab:llm-hp}} compiles the results of a comprehensive sweep across tasks, base models and learning rates, comparing NB-LoRA to LoRA, DoRA, and PiSSA in terms of their robustness to these variations (see table caption for details). While different methods were competitive for different particular scenarios, when averaging across models and tasks NB-LoRA is clearly superior.

\begin{wraptable}{R}{0.56\linewidth}
    \centering
    \vskip -0.16in
    \caption{\textbf{Scalability to larger models:} Comparison of LoRA, PiSSA and NB-LoRA on LLaMA-3-70B with learning rates from 5e-5 to 5e-4. }
    \label{tab:70b}
    \begin{center}
    \scalebox{0.95}{
    \setlength\tabcolsep{3pt}
    \begin{tabular}{c|ccc|cc}
        \toprule
        \midrule
        \multirow{2}{*}{Method} & \multicolumn{3}{c|}{Learning Rate} & \multicolumn{2}{c}{Computation} \\
         & 5e-5 & 1e-4 & 5e-4 & GPU Mem. & Train Time\\
        \midrule \midrule
        LoRA & 86.2 & 86.2 & failed & 65.57GB & 169m\\
        PiSSA & 83.6 & 79.0 & 41.8 & 65.57GB & 170m\\
        NB-LoRA & \textbf{87.1} & 85.4 & 83.3 & 69.15GB & 185m\\
        \bottomrule
    \end{tabular}
    }
    \end{center}
    \vskip -0.1in
\end{wraptable}
\paragraph{Scalability to Larger Models.} We compared NB-LoRA to LoRA and PiSSA on the LLaMA-3-70B model for GSM8K and compared them in terms of computational resources, accuracy, and learning-rate robustness. In \cref{tab:70b} it can be seen that NB-LoRA achieved the highest accuracy overall. It uniformly outperformed PiSSA, while standard LoRA achieved good performance for low learning rates but was unstable for larger learning rates. NB-LoRA required slightly more computational resources than LoRA and PiSSA: $\sim$6\% more memory and $\sim$9\% longer training time.

\paragraph{Comparison with DeLoRA.} \cref{fig:delora} compares NB-LoRA with DeLoRA \cite{bini2025decoupling} with  $\delta$ set to 10, 20, and free (see \cref{sec:theory} and \cref{fig:nblora-delora} for discussion). Firstly, in panel (a) we see that NB-LoRA achieves the highest test accuracy, outperforming DeLoRA with an equivalent norm bound $\delta=10$ by about 10\%. Secondly, in (b) we see that the NB-LoRA parameterization achieves very tight norm bounds, whereas with DeLoRA they are quite loose: with $\delta=10$, the observed Frobenius norm only reached around 2.5. With $\delta$ free, the Frobenius norm grew beyond the norm bound. Panel (c) shows that NB-LoRA achieves lower loss and higher gradient norm than DeLoRA. Note that the higher gradient norm explains the apparent ``offset'' in panel (a) w.r.t. learning rate.

\begin{figure}[!tb]
    \centering
    \begin{subfigure}[b]{0.48\textwidth}
        \centering
        \includegraphics[width=\linewidth]{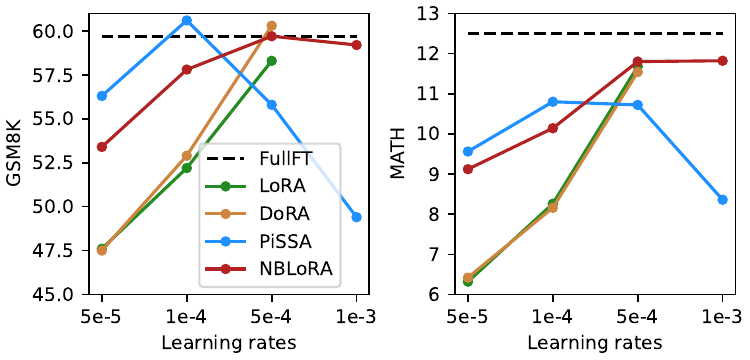}
        \caption{Test accuracy for math problem}
    \end{subfigure}
    ~
    \begin{subfigure}[b]{0.48\textwidth}
        \centering
        \includegraphics[width=\linewidth]{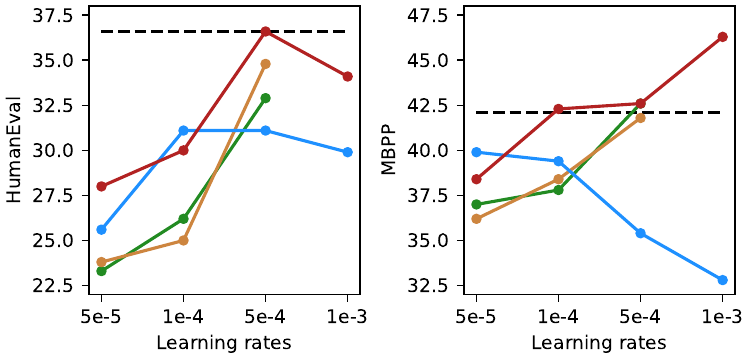}
        \caption{Test accuracy for python coding}
    \end{subfigure}
    ~
    \begin{subfigure}[b]{0.48\textwidth}
        \centering
        \includegraphics[width=\linewidth]{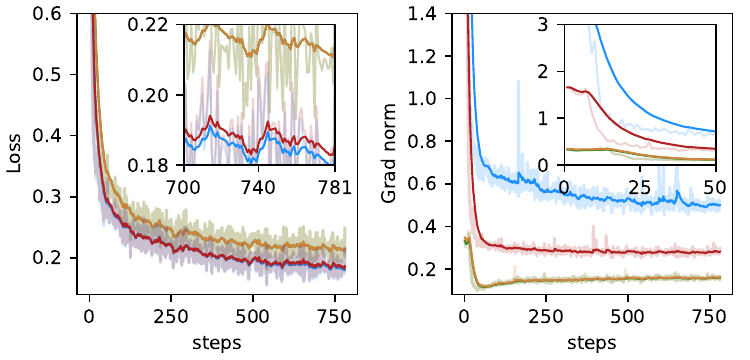}
        \caption{Learning rate of 5e-5}
    \end{subfigure}
    ~
    \begin{subfigure}[b]{0.48\textwidth}
        \centering
        \includegraphics[width=\linewidth]{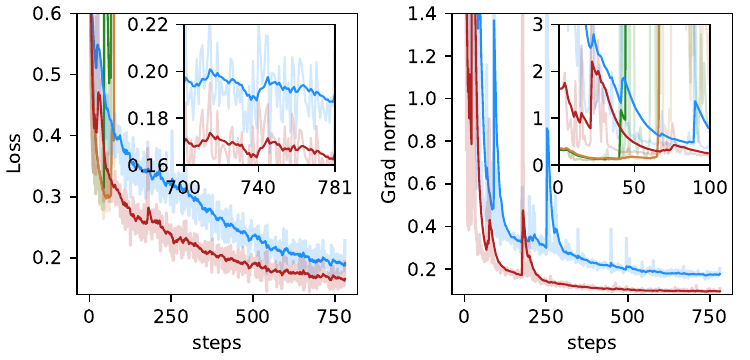}
        \caption{Learning rate of 1e-3}
    \end{subfigure}
    \caption{(Top) the evaluation accuracy over a range of learning rates and (Bottom) the loss and grad norm over the training steps for two learning rates.}
    \label{fig:llm-lr}
\end{figure}

\begin{table}[!tb]
    \centering
    \caption{Fine-tuning three base models based on LoRA (Lo), DoRA (Do), PiSSA (Pi) and NB-LoRA (NB) over different learning rates (\{1e-5, 5e-5,1e-4\} for Mistral and \{5e-5, 1e-4, 5e-4\} for LLaMA). We report the minimum, maximum and averaged test results, where the metrics for math and coding are $\frac{1}{2}(\mathrm{GSM8K}+\mathrm{MATH})$ and $\frac{1}{2}(\mathrm{HumanEval}+\mathrm{MBPP})$, respectively.}
    \label{tab:llm-hp}
    \begin{center}
    \scalebox{0.98}{
    \setlength\tabcolsep{3.5pt}
    \begin{tabular}{c|c|cccc|cccc|cccc||cccc}
        \toprule
        \midrule
         \multicolumn{2}{c|}{Base Model} &\multicolumn{4}{c|}{Mistral-7B-v0.1} &  \multicolumn{4}{c|}{LLaMA-3-8B} & \multicolumn{4}{c||}{LLaMA-2-13B} & \multicolumn{4}{c}{Model Avg.} \\
         \midrule \midrule
          \multicolumn{2}{c|}{Method} & Lo & Do & Pi & NB & Lo & Do & Pi & NB & Lo & Do & Pi & NB & Lo & Do & Pi & NB \\
          \midrule 
          \multirow{3}{*}{Math} 
          & min 
          & 43.8 & 44.2 & 45.9 & 46.7 
          & 49.6 & 50.4 & 41.6 & 48.5
          & 35.1 & 35.1 & 38.2 & 38.1
          & 42.8 & 43.2 & 41.9 & \textbf{44.4}
          \\
          & max 
          & 48.2 & 48.2 & 47.0 & 48.0
          & 51.5 & 51.8 & 52.0 & 52.9
          & 41.7 & 41.0 & 40.5 & 41.3
          & 47.1 & 47.0 & 46.5 & \textbf{47.4}
          \\
          & avg
          & 46.6 & 46.4 & 46.5 & 47.4
          & 50.7 & 51.3 & 48.2 & 51.2
          & 38.1 & 38.0 & 39.3 & 39.7
          & 45.1 & 45.2 & 44.7 & \textbf{46.1}
          \\
          \midrule
          \multirow{3}{*}{Code} 
            & min 
            & 52.4 & 53.7 & 55.6 & 57.8
            & 59.5 & 61.1 & 44.4 & 59.8
            & 42.5 & 42.5 & 44.4 & 44.0
            & 51.5 & 52.4 & 48.1 & \textbf{53.9}
            \\
            & max 
            & 58.3 & 59.2 & 59.0 & 59.7
            & 63.2 & 62.6 & 63.0 & 63.1
            & 46.6 & 45.9 & 45.6 & 48.8
            & 56.0 & 55.9 & 55.9 & \textbf{57.2}
            \\
            & avg 
            & 56.2 & 56.5 & 57.0 & 58.8
            & 61.6 & 61.8 & 55.0 & 61.9
            & 44.7 & 44.3 & 45.2 & 47.0
            & 54.2 & 54.2 & 52.4 & \textbf{55.9}
            \\
            \midrule \midrule
            \multirow{3}{*}{Task Avg.} 
            & min 
            & 48.1 & 49.0 & 50.8 & \textbf{52.2} 
            & 54.5 & \textbf{55.8} & 43.0 & 54.1 
            & 38.8 & 38.8 & \textbf{41.3} & 41.0
            & 47.1 & 47.8 & 45.0 & \textbf{49.1}
            \\
            & max 
            & 53.2 & 53.7 & 53.0 & \textbf{53.9} 
            & 57.4 & 57.2 & 57.5 & \textbf{58.0} 
            & 44.2 & 43.5 & 43.0 & \textbf{45.0}
            & 51.6 & 51.5 & 51.2 & \textbf{52.3}
            \\
            & avg 
            & 51.4 & 51.5 & 51.8 & \textbf{53.1} 
            & 56.2 & \textbf{56.5} & 51.6 & \textbf{56.5} 
            & 41.4 & 41.1 & 42.2 & \textbf{43.4}
            & 49.7 & 49.7 & 48.5 & \textbf{51.0}
            \\
        \bottomrule
    \end{tabular}
    }
    \end{center}
    \vskip -0.1in
\end{table}

\begin{figure}[!tb]
    \centering
    \begin{subfigure}[b]{0.24\textwidth}
        \centering
        \includegraphics[width=\linewidth]{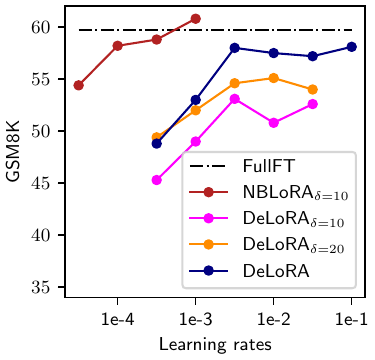}
        \caption{Test accuracy}
    \end{subfigure}
    ~
    \begin{subfigure}[b]{0.24\textwidth}
        \centering
        \includegraphics[width=\linewidth]{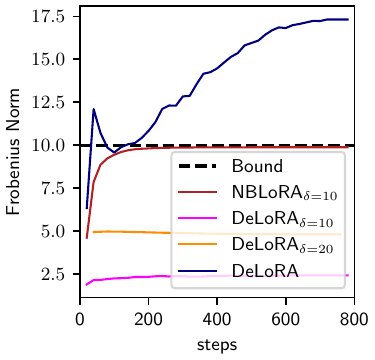}
        \caption{$\|W\|_{F}$ }
    \end{subfigure}
    ~
    \begin{subfigure}[b]{0.48\textwidth}
        \centering
        \includegraphics[width=\linewidth]{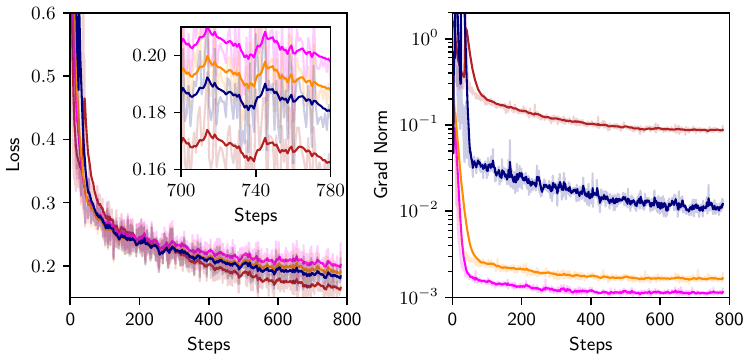}
        \caption{Loss and grad norm}
    \end{subfigure}
    \caption{(a) Comparison of DeLoRA and NB-LoRA over a range of learning rates; (b) Frobenius norm bound (maximized over all adaptation modules); (c) loss and grad norm vs training steps.}
    \label{fig:delora}
\end{figure}

\paragraph{Computation Cost.} \cref{tab:llm-computation} compares computational costs against two methods, PiSSA and DoRA, across different ranks on LLaMA 2-7B. Due to the extra reparameterization layer, NB-LoRA takes slightly more GPU memory and training time. NB-LoRA with the largest rank $r=256$ still takes less computational resources than DoRA with the smallest rank $r=2$. The main reason is that DoRA requires explicit calculation of the full adaptation matrix, which can be avoided with PiSSA and NB-LoRA. 

\begin{table}[!tb]
    \centering
    \caption{Computation comparison of DoRA, PiSSA and NB-LoRA with rank choice from 2 to 256. Experiments are conducted with 4 H200 GPUs. }
    \label{tab:llm-computation}
    \begin{center}
    \scalebox{0.98}{
    \setlength\tabcolsep{3.5pt}
    \begin{tabular}{c|c|cccccccc}
        \toprule
        \midrule
        \multicolumn{2}{c|}{Rank}  
        & 2 & 4 & 8 & 16 & 32 & 64 & 128 & 256 \\
        \midrule \midrule
        \multirow{3}{*}{Training Time} & DoRA 
        & 24m46s& 24m33s& 24m03s& 24m04s& 24m05s& 24m02s & 24m18s& 24m53s\\
        & PiSSA 
        & 17m44s& 17m35s& 17m09s& 17m10s& 17m12s& 17m15s& 17m32s& 18m09s\\
        & NB-LoRA 
        & 18m53s& 18m55s& 18m26s& 18m38s& 	18m51s& 19m13s& 20m15s& 22m40s\\
        \midrule
        \multirow{3}{*}{Peak GPU Mem. (GB)} & DoRA 
        & 102.41& 102.44& 102.51& 102.64& 102.90& 103.41&	104.47& 106.53\\
        & PiSSA 
        & 60.92& 60.96& 61.02& 61.15& 61.41& 61.93& 62.96& 65.04\\
        & NB-LoRA 
        &60.94	&60.99	&61.08	&61.28	&61.67	&62.45	&64.05	&67.19 \\
        \bottomrule
    \end{tabular}
    }
    \end{center}
    \vskip -0.1in
\end{table}

\section{ViT Experiments}\label{sec:vit}

To further explore the potential benefits of NB-LoRA, we conducted experiments in fine-tuning a vision transformer (ViT) model. In particular, we explore adaptation performance to a target dataset vs forgetting of the source (pretraining) dataset as well as hyperparameter robustness. We compare to standard LoRA and several recently-proposed methods. 

\paragraph{Adaptation vs Forgetting}
The main goal of this experiment is to explore the utility of norm bounds in preventing catastrophic model forgetting \cite{mccloskey1989catastrophic,french1999catastrophic,wang2024comprehensive}. Our hypothesis is that tight control of the adaption norm will prevent loss of performance on the pre-trained model as per the analysis in \cref{sec:lora}, while still enabling good adaptation performance. We perform experiments \citep{bafghi2024parameter} on ViT-B/16 model \citep{dosovitskiy2020image}, which is pre-trained on ImageNet-21k \citep{deng2009imagenet} and then fine-tuned to ImageNet-1k. For the proposed NB-LoRA, we choose the norm bound as $\delta =\gamma \|W_p\|_{S_p}$, where the ratio $\gamma$ is a hyper-parameter. Similar to the setup in \cite{kopiczko2024elora}, we adapt $Q,V$ matrices and learn the classification head for the Street View House Number (SVHN) dataset \citep{Netzer2011}. Here we report the results for NB-LoRA using nuclear norm with bound ratio of $\gamma$ between 0.1 and 1.6, see \cref{sec:vit_appendix} for additional results with different setups and datasets including  CIFAR-100 \cite{krizhevsky2009learning} and Food-101 \cite{bossard2014food}. 

\begin{figure}[!tb]
    \centering
    \includegraphics[width=0.95\linewidth]{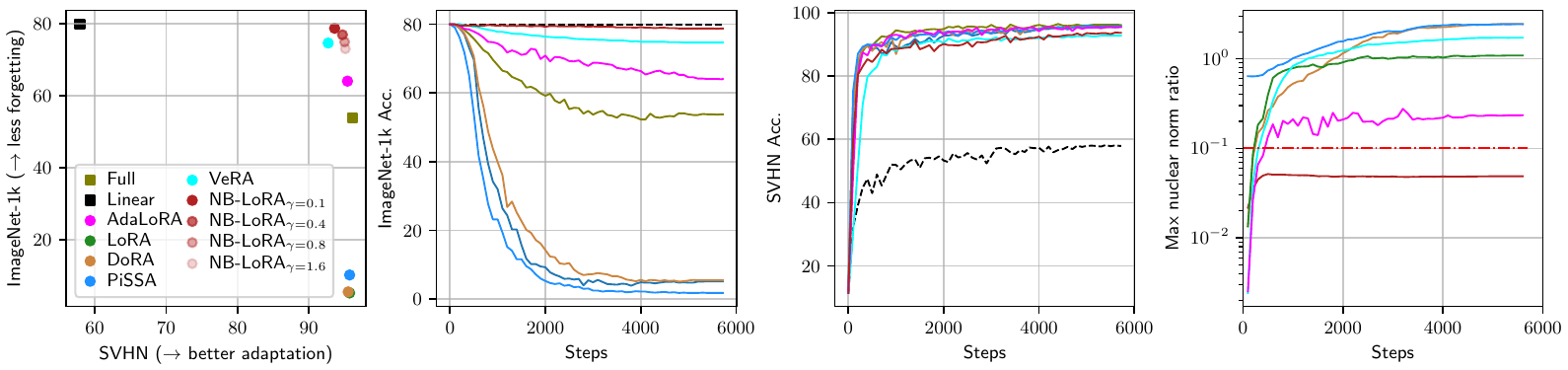}
    \caption{Analysis of Adaptation to a target data set vs forgetting of a source dataset. Left: final results for a variety of methods. Middle-Left: forgetting of the source dataset (ImageNet-1k). Middle-Right: adaptation to the target dataset (SVHN). Right: maximum nuclear norm ratio observed during training.}
    \label{fig:svhn}
\end{figure}

The \textbf{metric for model forgetting} is the test accuracy of the fine-tuned model on the source dataset: ImageNet-1k, which can be compared against performance on the target dataset. As shown in \cref{fig:svhn}(Left), the linear adapter (i.e. just learning the classification head) avoids forgetting of the source but has poor performance on the target set. In contrast, LoRA, DoRA and PiSSA achieve high adaptation performance to the target data set, but with a dramatic loss of performance on the source data set (from around 80\% to less than 10\%). AdaLoRA achieves good target adaptation with less severe but still significant forgetting. VeRA and NB-LoRA can both achieve a good balance of both, but NB-LoRA outperforms in terms of both source and target performance. It can also be seen that tuning of $\gamma$ allows a trade-off between source and target performance.

The middle panels of \cref{fig:svhn} show the evolution of source and target accuracy vs training steps. All models (except linear) perform quite similarly in terms of adaptation to the target, whereas on the source dataset NB-LoRA (shown with $\gamma=0.1$) maintains high accuracy throughout training, while most other methods quickly forget source performance. For PiSSA, DoRA, and LoRA the source performance drops significantly before target accuracy has converged, so early stopping can not solve the problem. \cref{fig:svhn} (Right) plots the maximum nuclear norm ratio of the models. NB-LoRA remains below the bound, while several others are more than an order of magnitude larger.

\begin{figure}[!tb]
    \centering
    \includegraphics[width=0.75\linewidth]{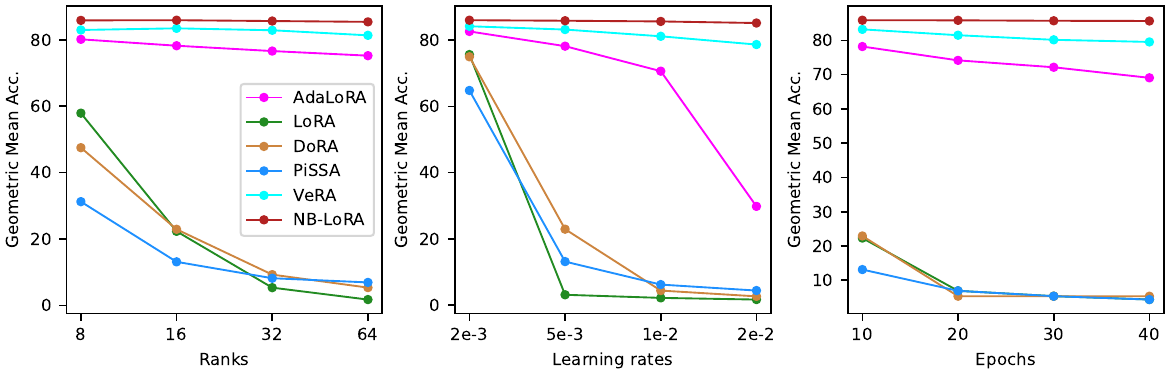}
    \caption{Analysis of hyperparameter robustness of different methods in terms of geometric mean of source (ImageNet-1k) and target (SVHN) dataset accuracies.}
    \label{fig:svhn-hp}
\end{figure}

\cref{fig:svhn-hp} shows an analysis of hyperparameter robustness of the different models, in terms of geometric mean of source and target accuracy. It can be seen that NB-LoRA is almost invariant with respect to these hyper-parameters, while most other methods (except VeRA) are highly sensitive to some or all of them.

\section{Conclusion}
In this paper we propose a norm-bounded low-rank adaptation (NB-LoRA) for model fine tuning. In particular, we introduce a new parameterization which is smooth and complete, i.e. it covers all matrices of a specified rank and singular value bounds, which can then be used to impose a Schatten $p$-norm bound (e.g. Frobenius, nuclear, spectral norm).

We argue that the proposed parameterization mitigates addresses some challenges related to the initialization of LoRA and its impact on learning rate, and can also mitigate the tendency of LoRA to forget source model performance. In experiments on fine tuning of language models, we compare to standard LoRA and other existing methods and demonstrate that NB-LoRA can substantially improve overall performance and robustness to learning rate. We showed that NB-LoRA is scalable to larger models (LLaMa-3-70B) with only a moderate computational penalty relative to standard LoRA.

\bibliographystyle{unsrtnat}
\bibliography{references}  

\newpage
\appendix
\section{Key Technical Lemmas}\label{sec:proof}

Here we present some key lemmas which are used in our proofs later.
\begin{lemma}\label{lem:uv}
    For any $Q\in \R^{n\times n}$, there exists a diagonal matrix $P$ with $P_{jj}\in \{-1,1\}$ such that $I+PQ^\top$ is invertible.
\end{lemma}
\begin{proof}
    Let $e_k, q_k$ be the $k$th column of $I$ and $Q$, respectively. We construct $A_k$ via
    \begin{equation}
        A_{k}^{-1}=A_{k-1}^{-1}-\frac{s_kA_{k-1}^{-1}e_k q_k^\top A_{k-1}^{-1}}{1+s_kq_k^\top A_{k-1}^{-1}e_k},
    \end{equation}
    where $A_0=I$ and $s_k=\mathrm{sign}\bigl(v_k^\top A_{k-1}^{-1}e_k\bigr)$ with $\mathrm{sign}(0)=1$. From Sherman-Morrison formula, $A_{k}$ is well-defined (i.e., invertible) and satisfies $A_{k}=A_{k-1}+ s_ke_k q_k^\top$. By taking $P=\mathrm{diag}(s_1,\ldots,s_n)$, we have $A_{n}= I +\sum_{k=1}^{n}s_k e_kq_k^\top =I+PQ^\top$ is also invertible. 
\end{proof}

\begin{lemma}\label{lem:cayley}
    Let $G\in \R^{r\times r}$ and $H\in \R^{s\times r}$ such that $G^\top G+H^\top H=I$. Then, 
    \begin{equation}\label{eq:cayley-appendix}
        \begin{bmatrix}
        G \\ H
    \end{bmatrix}=\cayley\left(\begin{bmatrix}
        X \\ Y
    \end{bmatrix}\right)=\begin{bmatrix}
        (I-Z)(I+Z)^{-1} \\ -2Y(I+Z)^{-1}
    \end{bmatrix}
    \end{equation}
    for some $X\in \R^{r\times r}$ and $Y\in \R^{s\times r}$ if and only if $I+G$ is invertible.
\end{lemma}
\begin{proof}
    From the Cayley transformation (\ref{eq:cayley}) we have the following relationships: 
    \begin{equation}\label{eq:xyz}
        G=(I-Z)(I+Z)^{-1},\quad H=-2Y(I+Z)^{-1},\quad Z=X-X^\top + Y^\top Y.
    \end{equation}

    (\textbf{if}). From the above equation we have $I+G=(I+Z)^{-1}$ invertible.

    (\textbf{only if}). The proof is constructive, i.e., finding $X, Z\in \R^{r\times r}$ and $ Y\in \R^{s\times r}$ satisfying (\ref{eq:xyz}). We consider a candidate solution as follows:
    \begin{equation}\label{eq:cayley-inverse}
        Z=(I+G)^{-1}(I-G),\quad Y=-\frac{1}{2}H(I+Z),\quad X=\frac{1}{2}Z.
    \end{equation}
    It is easy to check that the above solution satisfies the first two equations in (\ref{eq:xyz}). We now verify the last equation as follows:
    \begin{equation*}
        \begin{split}
            &Z+X^\top -X - Y^\top Y=\frac{1}{2}(Z+Z^\top)-Y^\top Y \\
            =&\frac{1}{2}[(I+G)^{-1}(I-G)+(I-G^\top)(I+G^\top)^{-1}]- (I+G^\top)^{-1}H^\top H(I+G)^{-1} \\
            =&\frac{1}{2}[(I-G)(I+G)^{-1}+(I+G^\top)^{-1}(I-G^\top)]- (I+G^\top)^{-1}H^\top H(I+G)^{-1}\\
            = & \frac{1}{2}(I+G^\top)^{-1}[(I+G^\top)(I-G)+(I-G^\top)(I+G)-2H^\top H](I+G)^{-1} \\
            =&(I+G^\top)^{-1}[I-G^\top G-H^\top H](I+G)^{-1}=0,
        \end{split}
    \end{equation*}
    where the second line is due to that $(I+G)^{-1}$ and $(I-G)$ are commutative.  
\end{proof}

\begin{lemma}\label{lem:cayley-inverse}
    Let $A\in \R^{r\times m}$ and $B\in \R^{r\times n}$ with $AA^\top + BB^\top = I$. Then, there exist a diagonal matrix $P\in \R^{r\times r}$ with $P_{jj}\in \{-1,1\}$ and $\tilde{A}\in \R^{r\times m},\tilde{B}\in \R^{r\times n}$ satisfying
    \begin{equation}
        \begin{bmatrix}
            PA & PB
        \end{bmatrix}^\top =\cayley\left(\begin{bmatrix}
            \tilde{A} & \tilde{B}
        \end{bmatrix}^\top\right).
    \end{equation}
\end{lemma}
\begin{proof}
    From the assumption we have that $\begin{bmatrix}
        A^\top \\ B^\top
    \end{bmatrix}$ is a tall matrix, i.e., $r\leq m+n$. We then take the partition $\begin{bmatrix}
        A^\top \\ B^\top
    \end{bmatrix}=\begin{bmatrix}
        \bar{G} \\ \bar{H}
    \end{bmatrix}$ with $\bar{G}\in \R^{r\times r}$ and $\bar{H}\in \R^{(m+n-r)\times r}$. We introduce $G=\bar{G}P$ and $H=\bar{H}P$, where $P$ is a diagonal matrix with $P_{jj}\in \{-1,1\}$. Then, we can obtain
    \[
    G^\top G+H^\top H=P(\bar{G}^\top \bar{G}+\bar{H}^\top\bar{H})P=P(AA^\top +BB^\top)P=P^2=I
    \]
    for all diagonal such $P$. From \cref{lem:uv}, we can pick a particular $P$ such that $I+G= I+ \bar{G}P$ is invertible. We then follow \cref{lem:cayley} to compute $X\in \R^{r\times r}$ and $Y\in \R{(m+n-r)\times r}$ satisfying (\ref{eq:cayley-appendix}). Finally, we take the partition $\begin{bmatrix}
        X^\top & Y^\top
    \end{bmatrix}=\begin{bmatrix}
        \tilde{A} & \tilde{B}
    \end{bmatrix}$.
\end{proof}

\section{Proof of \cref{thm:param}}\label{sec:thm-proof}

\begin{figure}[!ht]
    \centering
    \includegraphics[width=0.98\linewidth]{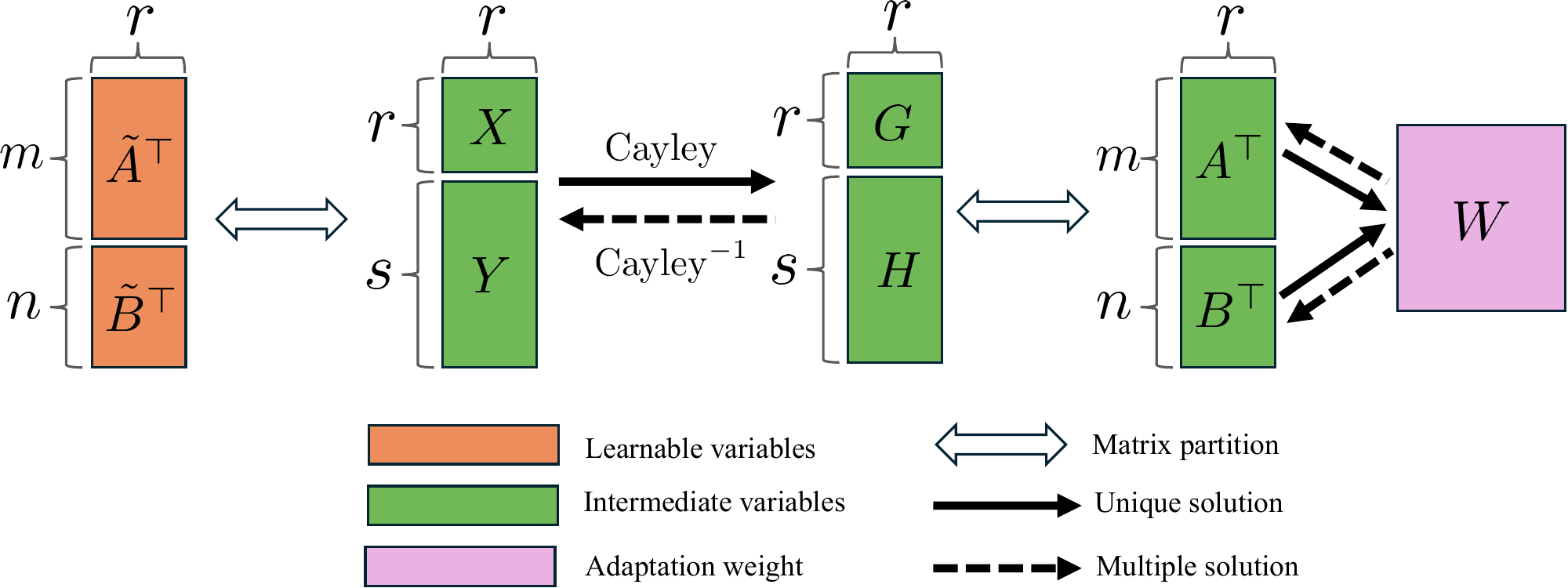}
    \caption{Diagram of NB-LoRA parameterization. }
    \label{fig:diagram}
\end{figure}
The proof includes two parts: I) $W=\gW(\tilde{A},\tilde{B})\in \sW_{S}$
for any $\tilde{A}\in \R^{r\times m}$ and $\tilde{B}\in \R^{r\times n}$; II) for any $W\in \sW_{S}$, there exists a pair of $\tilde{A}\in \R^{r\times m}$ and $\tilde{B}\in \R^{r\times n}$ such that $W=\gW(\tilde{A},\tilde{B})$. 

\paragraph{Part I.} It is obvious that $\rank(W)\leq r$. The $j$th singular value of $W$ satisfies
\begin{equation}\label{eq:magm}
\sigma_j( W)=2\sigma_j(\underset{Q^\top}{\underbrace{B^\top S^{\frac{1}{2}}}}\underset{K}{\underbrace{S^{\frac{1}{2}}A}})\leq \sigma_j\left(QQ^\top + KK^\top\right)=\sigma_j(S^{\frac{1}{2}}(\underset{I}{\underbrace{AA^\top +BB^\top}})S^{\frac{1}{2}})=\sigma_j(S)
\end{equation} 
where the inequality is the matrix arithmetic-geometric mean inequality \citep{bhatia1990singular,bhatia2013matrix}, and the last equality follows by the Cayley transformation.

\paragraph{Part II.} Without loss of generality, we assume that the diagonal elements of $S$ is in descending order, i.e., $\sigma_j(S)=S_{jj}$ for $j=1,\ldots,r$. Since $W$ has maximally $r$ non-zero singular values, we can take the reduced SVD decomposition $W=U_w\Sigma_w V_w^\top$ where $U_w\in \R^{m\times r}, V_w\in \R^{n\times r}$ are semi-orthogonal, and the positive diagonal matrix $S_w\in \R^{r\times r}$. We now consider the following candidates for $A,B$:
\begin{equation}\label{eq:ab-candidate}
    A=P\Sigma_a V_w^\top,\quad B=P\Sigma_b U_w^\top,
\end{equation}
where $P\in \R^{r\times r}$ is a diagonal matrix with $P_{jj}\in \{-1,1\}$, and $\Sigma_a, \Sigma_b\in \R^{r\times r}$ are positive diagonal matrices. The first constraint for $A$ and $B$ is that $\begin{bmatrix}
    A & B
\end{bmatrix}^\top$ is semi-orthogonal since it is an output of the Cayley transformation. Thus, we have
\begin{equation}\label{eq:square-ab}
    I=AA^\top+BB^\top = P(\Sigma_a^2+\Sigma_b^2)P^\top \;\Longrightarrow \;\Sigma_a^2+\Sigma_b^2=I.
\end{equation}
The second constraint for $A,B$ is $W=2B^\top S A$, which implies
\begin{equation}\label{eq:multi-ab}
    U_w\Sigma_w V_w^\top=2U_w\Sigma_aP^\top SP \Sigma_b V_w^\top=U_w(2\Sigma_a\Sigma_bS)V_w^\top\;\Longrightarrow\; 2\Sigma_a \Sigma_b =\Sigma_wS^{-1}
\end{equation}
Eq.~(\ref{eq:square-ab}) and (\ref{eq:multi-ab}) yield a solution of
\begin{equation}\label{eq:sigma-ab}
    \Sigma_a=\frac{\sqrt{I+J}+\sqrt{I-J}}{2},\; \Sigma_b=\frac{\sqrt{I+J}-\sqrt{I-J}}{2}.
\end{equation} 
where $J=\Sigma_w S^{-1}$ satisfies $0\preceq J \preceq I$ since  $S_w\preceq S$ for $W\in \sW_S$. Note that we need to deal with the case where $S$ is not full rank, i.e., there exists an $k< r$ such that $S_{kk}=0$ and $S_{k-1,k-1}>0$. Since $0\preceq \Sigma_w\preceq S$, we have $\Sigma_{ii} =0$ for all $i\geq k$ and simply take $J_{ii}=1$. It is easy to verify that Equations~(\ref{eq:square-ab}) - (\ref{eq:sigma-ab}) still hold. Finally, \cref{lem:cayley-inverse} shows that we can recover $\tilde{A},\tilde{B}$ from $A,B$ by picking a proper $P$ in (\ref{eq:ab-candidate}) based on \cref{lem:uv}.

\section{Custom Backward for Cayley Transformation}\label{sec:cayley-backward}

We first rewrite the forward computation of Cayley transformation $(X,Y)\rightarrow (G, H)$ as follows:
\begin{equation}\label{eq:cayley-forward}
    Z=X-X^\top+Y^\top Y,\quad  M=I+Z,\quad W=M^{-1},\quad G=(I-Z)W,\quad H=-2YW
\end{equation}
where $G,X,Z,M,W\in \R^{r\times r}$ and $H,Y\in \R^{s\times r}$. We provide a custom backward $(\nabla_G,\nabla_H)\rightarrow(\nabla_X,\nabla_Y)$ with $\nabla_{A}=(\partial \ell/\partial A)^\top$ for (\ref{eq:cayley-forward}) as follows:
\begin{equation}\label{eq:cayley-backward}
    \begin{split}
        \begin{bmatrix}
    \tilde\nabla_G \\ \tilde\nabla_H
    \end{bmatrix}&=
    \begin{bmatrix}
    \nabla_G \\ \nabla_H
    \end{bmatrix} W^\top,\quad S_Z=
    \begin{bmatrix}
    I+G \\ H
    \end{bmatrix}^\top
    \begin{bmatrix}
    \tilde\nabla_G \\ \tilde\nabla_H
    \end{bmatrix},\\
    \nabla_X&=S_Z^\top-S_Z,\quad
    \nabla_Y=-\frac{1}{2}HM(S_Z^\top+S_Z)-2\tilde\nabla_H,
    \end{split}
\end{equation}
where $W,M,G,H$ can be reused from the Cayley forward step. Note that when applying AutoDiff to (\ref{eq:cayley-forward}), it is necessary to store the input $X,Y$, output $G,H$ as well as some intermediate steps, which requires more memory than our custom backward step (\ref{eq:cayley-backward}) since $Y\in\mathbb{R}^{s\times r}$ is much larger than $W,M\in\mathbb{R}^{r\times r}$. In our approach, we can recover $Y$ from other stored variables, i.e., $Y=-\frac{1}{2}HM$. To give detail derivation for (\ref{eq:cayley-backward}), we first differentiate the forward step (\ref{eq:cayley-forward}):
\begin{equation}\label{eq:cayley-diff}
    \begin{split}
        \diff Z &= \diff X -\diff X^\top + Y^\top \diff Y + \diff Y^\top Y,\quad 
        \diff W= -W \diff Z W, \\
        \diff G &= -\diff Z W +(I-Z)\diff W, \quad
        \diff H =-2\diff Y W -2Y \diff W.
    \end{split}
\end{equation}
The differential of loss function $\diff \ell$ satisfies
\begin{equation}\label{eq:loss-diff}
    \diff \ell= \Tr\left(\nabla_X^\top \diff X\right)+\Tr\left(\nabla_Y^\top \diff Y\right) = \Tr\left(\nabla_G^\top \diff G\right)+\Tr\bigl(\nabla_H^\top \diff H\bigr).
\end{equation}
By further substituting (\ref{eq:cayley-diff}) into (\ref{eq:loss-diff}), we have
\begin{equation*}
    \begin{split}
    \Tr\bigl(&\nabla_G^\top \diff G\bigr)+\Tr\bigl(\nabla_H^\top \diff H\bigr) \\
        &=-\Tr\bigl(\nabla_G^\top \diff Z W +\nabla_G^\top (I-Z)W\diff Z W\bigr) -2 \Tr\bigl(\nabla_H^\top \diff Y W - \nabla_H^\top Y W \diff Z W\bigr) \\
        &= -\Tr\bigl(W(\nabla_G^\top + \nabla_G^\top (I-Z)W -2\nabla_H^\top YW)\diff Z\bigr)-\Tr\bigl(2W\nabla_H^\top \diff Y\bigr)\\
        &= -\Tr\bigl(W(\nabla_G^\top + \nabla_G^\top G+ \nabla_H^\top H) \diff Z\bigr)-\Tr\bigl(2W\nabla_H^\top \diff Y\bigr)\\
        &=-\Tr\bigl(S_Z^\top \diff Z\bigr)-\Tr\bigl(2W\nabla_H^\top \diff Y\bigr)\\
        &=-\Tr\bigl((S_Z^\top-S_Z)dX\bigr)-\Tr\bigl(((S_Z+S_Z^\top)Y^\top+2W\nabla_H^\top)\diff Y\bigr)=\Tr\left(\nabla_X^\top \diff X\right)+\Tr\left(\nabla_Y^\top \diff Y\right)
    \end{split}
\end{equation*}
which yields the custom backward step (\ref{eq:cayley-backward}) by substituting $Y=-\frac{1}{2}HM$. As shown in \cref{tab:llm-custom-backward}, the custom backward pass can save both GPU memory and training time.

\begin{table}[!tb]
    \centering
    \caption{Computation comparison for training NB-LoRA with AutoDiff and custom backward step.}
    \label{tab:llm-custom-backward}
    \begin{center}
    \scalebox{0.9}{
    \setlength\tabcolsep{3.5pt}
    \begin{tabular}{cccc}
        \toprule
        \midrule
        Method & Peak Mem. & Train Time & GSM8K Acc. \\
        \midrule \midrule
        AutoDiff & 69.52GB & 23m51s & 58.0\\
        Custom & 67.19GB & 22m40s & 57.8 \\
        \bottomrule
    \end{tabular}
    }
    \end{center}
    \vskip -0.1in
\end{table}

\section{Connections between DeLoRA and NB-LoRA}\label{sec:comparison}

DeLoRA \citep{bini2025decoupling} is a fine-tuning method which can control both rank and Frobenius norm bound of weight adaptation $W$. Specifically, DeLoRA takes the form of
\begin{equation}\label{eq:delora}
    W=\frac{\delta}{r} B^\top \Xi A:=\frac{\delta}{r} B^\top \diag(|b_i|_2\cdot |a_i|_2) A
\end{equation}
where $a_i, b_i$ are the $i$th row of $A\in \R^{r\times n}$ and $ B\in \R^{r\times m}$, respectively. The above parameterization can be rewritten as sum of NB-LoRA matrices with both rank and norm bound of 1:
\[
W=\frac{\delta}{r}\sum_{i=1}^r 2 \left(\frac{b_i}{\sqrt{2}|b_i|_2}\right)^\top \left(\frac{a_i}{\sqrt{2}|a_i|_2}\right)=\frac{\delta}{r}\sum_{i=1}^{r}2\bar{b}_i^\top \bar{a}_i=\frac{\delta}{r}\sum_{i=1}^{r}\bar{W}_i,
\]
where $[\, \bar{a}_i \; \bar{b}_i\,]$ is a set of of decoupled unit vectors. By \cref{thm:param} we have that $\|\bar{W}_i\|_F\leq 1$ and $\|W\|_F\leq \delta /r \sum_{i=1}^{r}\|\bar{W}_i\|_F\leq \delta$. NB-LoRA in (\ref{eq:w-param}) also has a similar representation:
\[
W=2 B^\top S A=\sum_{i=1}^{r} s_i (2 \hat{b}_i^\top \hat{a}_i) = \sum_{i=1}^{r} s_i \hat{W}_i.
\]
Different from DeLoRA, $[\,\hat{a}_i \; \hat{b}_i\,]$ is a set of coupled unit vectors as they are orthogonal to each other. This coupling behavior allows us to specify the bound for each singular value of $W$, providing tight control of a wide family of matrix norms. Another main difference is model initialization. Since it is not straightforward to initialize $A, B$ satisfying $W=0$ for (\ref{eq:delora}), the residual-type initialization \citep{meng2024pissa} is adopted, resulting a smaller reachable set than NB-LoRA when an explicit bound is specified, see detailed discussion in \cref{sec:theory}. 

\section{LLM Experimental Details and Additional Results}\label{sec:llm-appendix}

\paragraph{Training Details.} In our LLM experiments, we use the same training setup as \cite{meng2024pissa,alpaca}, i.e., AdamW \cite{loshchilov2018decoupled} with no weight decay. We use the cosine annealing scheduler with a warm-up ratio of 0.03. The default batch size is 128. We ensure $\alpha=r$ for all adapters, although NB-LoRA does not use this parameter. We choose the norm bound of $\delta=r$ with nuclear norm, which results in the same scaling factor as the other adapters. When Frobenius or spectral norm is used, we set the default bound as $\delta=\sqrt{r}$ and $\delta=1$, respectively, which also results in the same scaling factor as other adapters. We set {\tt lora\_dropout} to 0, and insert the adapters into all linear layers of the base model. We use BFloat16 for both the base model and the adapters. 

\paragraph{Ablation of NB-LoRA Design Choice.} We summarize the incremental design choices that transform LoRA into NB-LoRA in \cref{tab:ablation}. We conduct an ablation study on LLaMA-2-7B fine-tuning in \cref{tab:ablation-experiment}. For a large rank ($r=128$), NB-LoRA with spectral norm bound achieves slightly better performance, whereas the nuclear norm performs better under a low-rank budget. Both methods yield more robust performance compared to LoRA.

\begin{table}[!tb]
    \centering
    \caption{Summary of incremental design choices from LoRA to NB-LoRA.}
    \label{tab:ablation}
    \begin{center}
    \scalebox{0.84}{
    \setlength\tabcolsep{3pt}
    \begin{tabular}{ll|l}
        \toprule
        \midrule
        Design choice & Method & $W$ formulation  \\
        \midrule \midrule
        &LoRA 
        & $\frac{\alpha}{r}B^\top A$	\\
        +(Cayley transform) & NB-LoRA with $\|W\|_{S_\infty}\leq \delta$  &  
        $2\delta B^\top A $ with $(A,B)=\mathrm{Cayley}(\tilde{A},\tilde{B})$\\
        +(learnable scaling) & NB-LoRA with $\|W\|_{S_p}\leq \delta$ & $2\delta B^\top S A$ with $S=\mathrm{diag}(s)$ and $\|s\|_p\leq \delta$\\
        \bottomrule
    \end{tabular}
    }
    \end{center}
    \vskip -0.1in
\end{table}

\begin{table}[!tb]
    \centering
    \caption{We report the GSM8K accuracy for ablation of NB-LoRA on fine-tuning LLaMA-2-7B models with different ranks and learning rates.}
    \label{tab:ablation-experiment}
    \begin{center}
    \scalebox{0.84}{
    \setlength\tabcolsep{3pt}
    \begin{tabular}{cc|ccc}
        \toprule
        \midrule
        \multirow{2}{*}{Rank} & \multirow{2}{*}{Method} & \multicolumn{3}{c}{Learning Rate}  \\
        & & 1e-4 & 5e-4 & 1e-3 \\
        \midrule \midrule
        \multirow{3}{*}{128}&LoRA 
        & 52.8& 58.3 & failed	\\
         & NB-LoRA (spectral)  & 57.7 & \textbf{60.5} & 60.0 
        \\
         & NB-LoRA (nuclear) & 57.8 & 59.7 & 59.2\\
        \midrule
         \multirow{3}{*}{16}&LoRA 
        & 43.2 & 55.8 & 57.5	\\
         & NB-LoRA (spectral)  & 47.9 & 55.3 & 56.5 
        \\
         & NB-LoRA (nuclear) & 49.4 & \textbf{56.8} & 55.6\\
        \bottomrule
    \end{tabular}
    }
    \end{center}
    \vskip -0.1in
\end{table}

\paragraph{Robust Performance for Prolong Training.} We conduct a full epoch training of LLaMA-2-7B on the MetaMathQA dataset. The learning rates are chosen to be 1e-4 and 5e-4, which achieve good performance for different adapters in short horizon training (see \cref{fig:llm-lr}). We can observe that NB-LoRA consistently outperforms other methods. In particular, NB-LoRA is more stable for the large learning rate, due to the norm saturation on weight adaptation. Meanwhile, DoRA depicts unstable training and PiSSA has poor performance due to excessive increase in weight norm. 

\paragraph{Experiments on Various Ranks.} \cref{fig:llm-rank} explores the impact of rank on LoRA, DoRA, PiSSA and NB-LoRA with learning rate of 1e-4. Under the setup, PiSSA achieves the best GSM8K accuracy. As the rank decrease, the gap between NB-LoRA and PiSSA narrows. And NB-LoRA outperforms PiSSA for low ranks when $r< 16$. NB-LoRA outperforms LoRA and DoRA by approximately 5\% across all ranks. We also examine the effect of varying learning rates at rank 16, demonstrating robustness to learning rate choices across different ranks. 

\paragraph{Addition Computation Comparison between DoRA and NB-LoRA.} We first report the forward computation time of key operations in DoRA and NB-LoRA in \cref{tab:llm-BA-inv}, showing that inverting a small low rank matrix is much computationally cheaper than computing a large low-ran weight matrix. 

\begin{table}[!tb]
    \centering
    \caption{Computation time ($\mu\mathrm{s}$) of the rank-$r$ matrix $B^\top A\in \R^{m\times n}$ in DoRA and $M^{-1}\in \R^{r\times r}$ in NB-LoRA. We use $m=4096$, $n=4m$ and rank $r$ from 2 to 256. Computation time is measured based on 500 samples with 500 warm-up steps on RTX4090.}
    \label{tab:llm-BA-inv}
    \begin{center}
    \scalebox{0.84}{
    \setlength\tabcolsep{3pt}
    \begin{tabular}{c|cccccccc}
        \toprule
        \midrule
        Matrix Operation & 2 & 4 & 8 & 16 & 32 & 64 & 128 & 256 \\
        \midrule \midrule
        $B^\top A \in \R^{m\times n}$ 
        & 314.1$\pm$2.5	& 311.9$\pm$1.6	& 312.1$\pm$2.4	& 310.7$\pm$2.0 & 288.3$\pm$2.0 & 323.9$\pm$1.9 & 421.9$\pm$2.0 & 792.6$\pm$1.1\\
        $M^{-1}\in \R^{r\times r}$ &  
        29.2$\pm$0.9&30.7$\pm$0.9&35.1$\pm$1.4&48.3$\pm$0.9&
        72.9$\pm$1.1&97.0$\pm$2.0&169.7$\pm$1.4&361.3$\pm$1.7\\
        \bottomrule
    \end{tabular}
    }
    \end{center}
    \vskip -0.1in
\end{table}

\begin{figure}[!tb]
    \centering
    \begin{subfigure}[b]{0.24\textwidth}
        \centering
        \includegraphics[width=\linewidth]{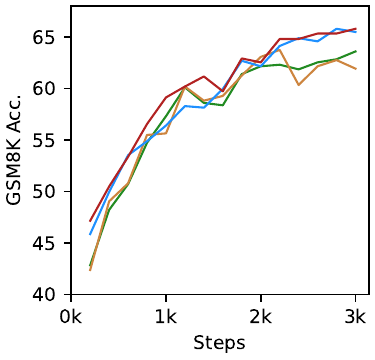}
        \caption{Test accuracy (lr=1e-4)}
    \end{subfigure}
    ~
    \begin{subfigure}[b]{0.24\textwidth}
        \centering
        \includegraphics[width=\linewidth]{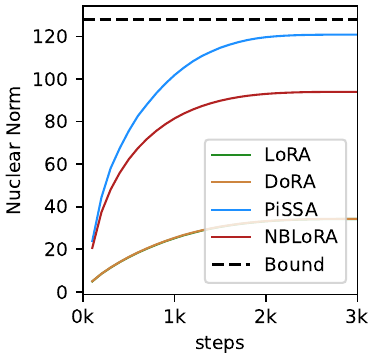}
        \caption{$\|W\|_{S_1}$ (lr=1e-4)}
    \end{subfigure}
    ~
    \begin{subfigure}[b]{0.48\textwidth}
        \centering
        \includegraphics[width=\linewidth]{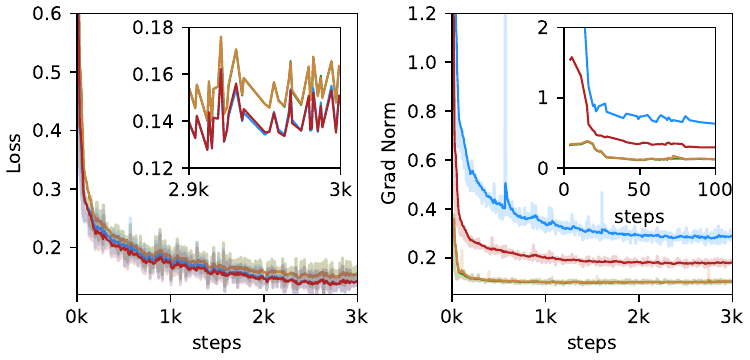}
        \caption{Loss and grad norm (lr=1e-4)}
    \end{subfigure}
    ~
    \begin{subfigure}[b]{0.24\textwidth}
        \centering
        \includegraphics[width=\linewidth]{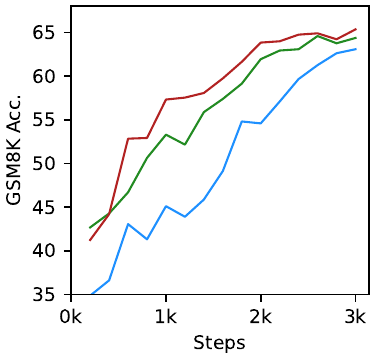}
        \caption{Test accuracy (lr=5e-4)}
    \end{subfigure}
    ~
    \begin{subfigure}[b]{0.24\textwidth}
        \centering
        \includegraphics[width=\linewidth]{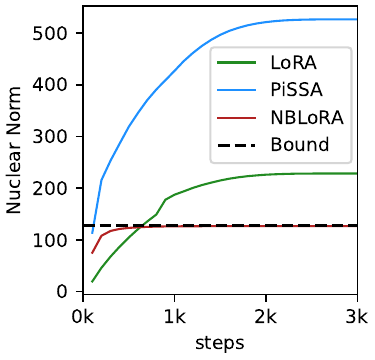}
        \caption{$\|W\|_{S_1}$ (lr=5e-4)}
    \end{subfigure}
    ~
    \begin{subfigure}[b]{0.48\textwidth}
        \centering
        \includegraphics[width=\linewidth]{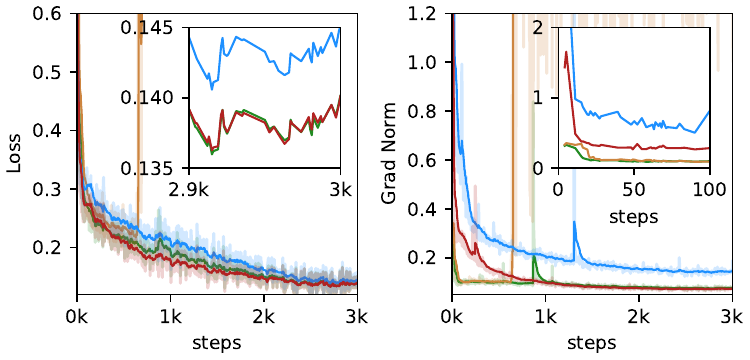}
        \caption{Loss and grad norm (lr=5e-4)}
    \end{subfigure}
    \caption{The evaluation accuracy, the nuclear norm bound, loss and grad norm over a full training epoch on MetaMathQA. The norm bound is computed by maximizing over all adaptation modules. }
    \label{fig:llm-full-epoch}
\end{figure}

\begin{figure}[!tb]
    \centering
    \includegraphics[width=\linewidth]{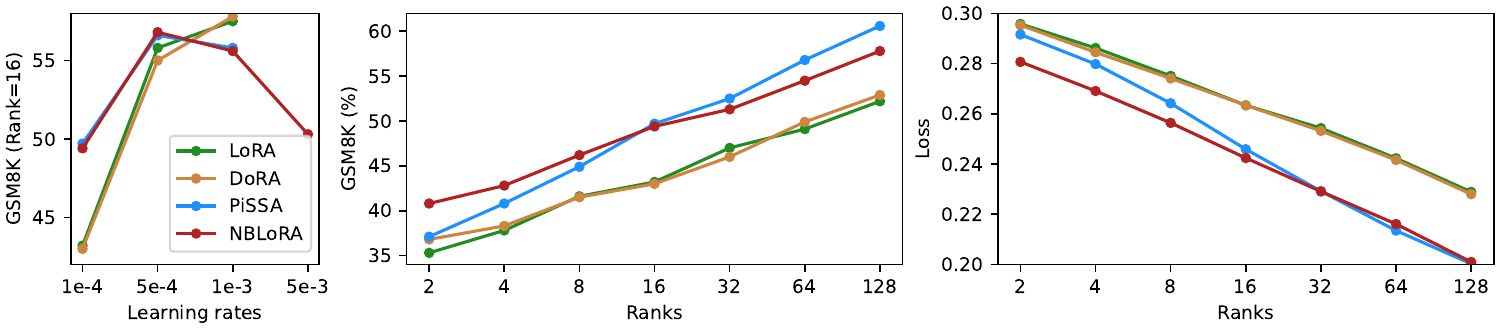}
    \caption{Comparison among LoRA, DoRA, PiSSA and NB-LoRA across ranks from 2 to 128. We also test the learning rate robustness for the case with $r=16$.}
    \label{fig:llm-rank}
\end{figure}

\newpage
\section{ViT Experiments}
\label{sec:vit_appendix}

\paragraph{Training Details.} A similar ViT fine-tuning experiments for the model forgetting issue can be found in \cite{bafghi2024parameter}. We take the ViT-B/16 model \cite{dosovitskiy2020image} and insert adaption blocks into the $Q, V$ matrices \cite{kopiczko2024elora}. We choose AdamW \cite{loshchilov2018decoupled} as the optimizer with default learning rate of $5\text{e-}3$ and weight decay of 0.01. For the full fine-tuning, we reduce the learning rate to $5\text{e-}4$. We take one-cycle learning rate scheduler with warm-up ratio of 0.1. We use batch size of 128 for SVHN dataset and 256 for CIFAR-100 and Food-101 dataset.

\paragraph{Extra results.} We report the ViT examples with different target datasets: CIFAR-100 and Food-101 in \cref{fig:hyper-parameter}. A similar conclusion as the SVHN experiment can be drawn from two datasets. 

\begin{figure}[!tb]
    \centering
    \begin{tabular}{c}
        \includegraphics[width=\linewidth]{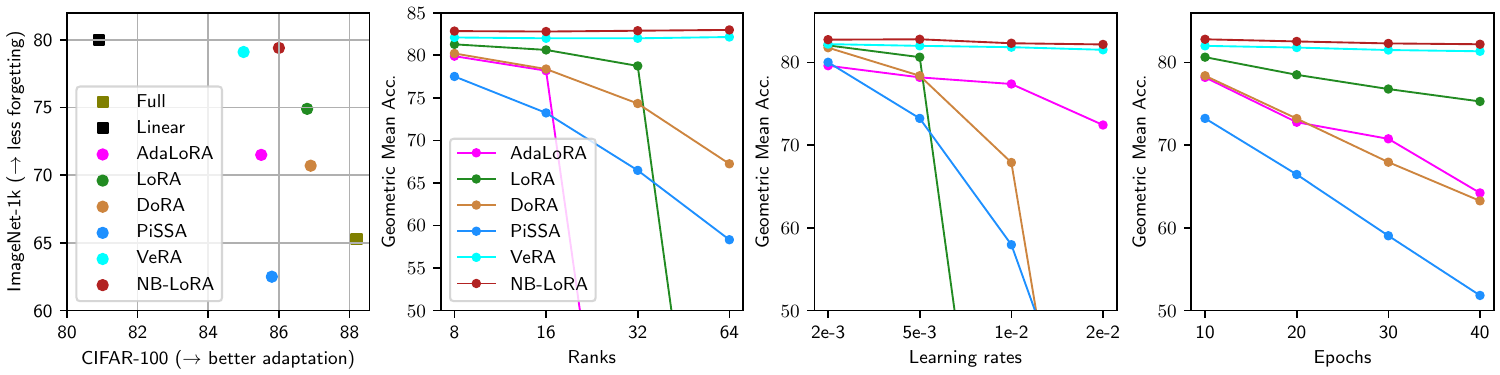}  \\
        \includegraphics[width=\linewidth]{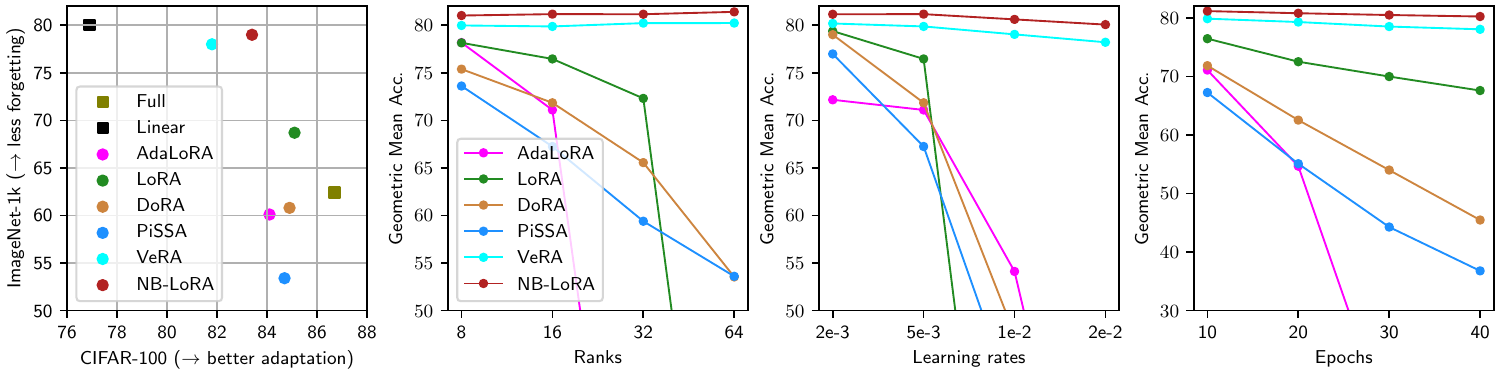} 
    \end{tabular}
    \caption{Geometric mean of CIFAR-100 (top) and Food-101 (bottom) with different adapters on various of hyper-parameter setup.}
    \label{fig:hyper-parameter}
\end{figure}

\end{document}

%% file: math_commands.tex
\usepackage{amsmath,amsfonts,bm}

\def\eqref#1{equation~\ref{#1}}

\def\1{\bm{1}}

\DeclareMathAlphabet{\mathsfit}{\encodingdefault}{\sfdefault}{m}{sl}
\SetMathAlphabet{\mathsfit}{bold}{\encodingdefault}{\sfdefault}{bx}{n}

\def\gW{{\mathcal{W}}}

\def\sW{{\mathbb{W}}}

\newcommand{\R}{\mathbb{R}}

\DeclareMathOperator{\diag}{diag}
\DeclareMathOperator{\rank}{rank}
\DeclareMathOperator{\cayley}{Cayley}
\newcommand{\diff}{\mathrm{d}}

\def\sW{{\mathbb{W}}}
\def\gW{{\mathcal{W}}}

\DeclareMathOperator{\Tr}{Tr}

%% file: main.bbl
\begin{thebibliography}{43}
\providecommand{\natexlab}[1]{#1}
\providecommand{\url}[1]{\texttt{#1}}
\expandafter\ifx\csname urlstyle\endcsname\relax
  \providecommand{\doi}[1]{doi: #1}\else
  \providecommand{\doi}{doi: \begingroup \urlstyle{rm}\Url}\fi

\bibitem[Achiam et~al.(2023)Achiam, Adler, Agarwal, Ahmad, Akkaya, Aleman, Almeida, Altenschmidt, Altman, Anadkat, et~al.]{achiam2023gpt}
Josh Achiam, Steven Adler, Sandhini Agarwal, Lama Ahmad, Ilge Akkaya, Florencia~Leoni Aleman, Diogo Almeida, Janko Altenschmidt, Sam Altman, Shyamal Anadkat, et~al.
\newblock Gpt-4 technical report.
\newblock \emph{arXiv preprint arXiv:2303.08774}, 2023.

\bibitem[Touvron et~al.(2023)Touvron, Lavril, Izacard, Martinet, Lachaux, Lacroix, Rozi{\`e}re, Goyal, Hambro, Azhar, et~al.]{touvron2023llama}
Hugo Touvron, Thibaut Lavril, Gautier Izacard, Xavier Martinet, Marie-Anne Lachaux, Timoth{\'e}e Lacroix, Baptiste Rozi{\`e}re, Naman Goyal, Eric Hambro, Faisal Azhar, et~al.
\newblock Llama: Open and efficient foundation language models.
\newblock \emph{arXiv preprint arXiv:2302.13971}, 2023.

\bibitem[Houlsby et~al.(2019)Houlsby, Giurgiu, Jastrzebski, Morrone, De~Laroussilhe, Gesmundo, Attariyan, and Gelly]{houlsby2019parameter}
Neil Houlsby, Andrei Giurgiu, Stanislaw Jastrzebski, Bruna Morrone, Quentin De~Laroussilhe, Andrea Gesmundo, Mona Attariyan, and Sylvain Gelly.
\newblock Parameter-efficient transfer learning for nlp.
\newblock In \emph{International conference on machine learning}, pages 2790--2799. PMLR, 2019.

\bibitem[Hu et~al.(2022)Hu, Wallis, Allen-Zhu, Li, Wang, Wang, Chen, et~al.]{hu2022lora}
Edward~J Hu, Phillip Wallis, Zeyuan Allen-Zhu, Yuanzhi Li, Shean Wang, Lu~Wang, Weizhu Chen, et~al.
\newblock Lora: Low-rank adaptation of large language models.
\newblock In \emph{ICLR}, 2022.

\bibitem[Qiu et~al.(2023)Qiu, Liu, Feng, Xue, Feng, Liu, Zhang, Weller, and Sch{\"o}lkopf]{qiu2023controlling}
Zeju Qiu, Weiyang Liu, Haiwen Feng, Yuxuan Xue, Yao Feng, Zhen Liu, Dan Zhang, Adrian Weller, and Bernhard Sch{\"o}lkopf.
\newblock Controlling text-to-image diffusion by orthogonal finetuning.
\newblock \emph{Advances in Neural Information Processing Systems}, 36:\penalty0 79320--79362, 2023.

\bibitem[Biderman et~al.(2024)Biderman, Portes, Ortiz, Paul, Greengard, Jennings, King, Havens, Chiley, Frankle, Blakeney, and Cunningham]{biderman2024lora}
Dan Biderman, Jacob Portes, Jose Javier~Gonzalez Ortiz, Mansheej Paul, Philip Greengard, Connor Jennings, Daniel King, Sam Havens, Vitaliy Chiley, Jonathan Frankle, Cody Blakeney, and John~Patrick Cunningham.
\newblock Lo{RA} learns less and forgets less.
\newblock \emph{Transactions on Machine Learning Research}, 2024.

\bibitem[Jang et~al.(2024)Jang, Lee, and Ryu]{jang2024lora}
Uijeong Jang, Jason~D Lee, and Ernest~K Ryu.
\newblock Lora training in the ntk regime has no spurious local minima.
\newblock In \emph{International Conference on Machine Learning}, 2024.

\bibitem[Kim et~al.(2025)Kim, Kim, and Ryu]{kim2025lora}
Junsu Kim, Jaeyeon Kim, and Ernest~K Ryu.
\newblock Lora training provably converges to a low-rank global minimum or it fails loudly (but it probably won't fail).
\newblock In \emph{International Conference on Machine Learning}, 2025.

\bibitem[Bini et~al.(2025)Bini, Girrbach, and Akata]{bini2025decoupling}
Massimo Bini, Leander Girrbach, and Zeynep Akata.
\newblock Delora: Decoupling angles and strength in low-rank adaptation.
\newblock In \emph{International Conference on Learning Representations}, 2025.

\bibitem[Hu et~al.(2025)Hu, Su, Kuo, Song, and Liu]{hu2025computational}
Jerry Yao-Chieh Hu, Maojiang Su, En-Jui Kuo, Zhao Song, and Han Liu.
\newblock Computational limits of low-rank adaptation (lora) fine-tuning for transformer models.
\newblock In \emph{ICLR 2025 Workshop on Deep Generative Model in Machine Learning: Theory, Principle and Efficacy}, 2025.

\bibitem[Bini et~al.(2024)Bini, Roth, Akata, and Khoreva]{bini2024ether}
Massimo Bini, Karsten Roth, Zeynep Akata, and Anna Khoreva.
\newblock Ether: Efficient finetuning of large-scale models with hyperplane reflections.
\newblock In \emph{ICML}, 2024.

\bibitem[Hayou et~al.(2024)Hayou, Ghosh, and Yu]{hayou2024impact}
Soufiane Hayou, Nikhil Ghosh, and Bin Yu.
\newblock The impact of initialization on lo{RA} finetuning dynamics.
\newblock In \emph{The Thirty-eighth Annual Conference on Neural Information Processing Systems}, 2024.
\newblock URL \url{https://openreview.net/forum?id=sn3UrYRItk}.

\bibitem[Gouk et~al.(2021)Gouk, Hospedales, et~al.]{gouk2021distance}
Henry Gouk, Timothy Hospedales, et~al.
\newblock Distance-based regularisation of deep networks for fine-tuning.
\newblock In \emph{International Conference on Learning Representations}, 2021.

\bibitem[Chen et~al.(2023)Chen, Zhang, Shi, Li, Smola, and Yang]{chen2023parameter}
Jiaao Chen, Aston Zhang, Xingjian Shi, Mu~Li, Alex Smola, and Diyi Yang.
\newblock Parameter-efficient fine-tuning design spaces.
\newblock In \emph{International Conference on Learning Representations}, 2023.

\bibitem[Liu et~al.(2024)Liu, Wang, Yin, Molchanov, Wang, Cheng, and Chen]{liu2024dora}
Shih-yang Liu, Chien-Yi Wang, Hongxu Yin, Pavlo Molchanov, Yu-Chiang~Frank Wang, Kwang-Ting Cheng, and Min-Hung Chen.
\newblock Dora: Weight-decomposed low-rank adaptation.
\newblock In \emph{Forty-first International Conference on Machine Learning}, 2024.

\bibitem[Kopiczko et~al.(2024{\natexlab{a}})Kopiczko, Blankevoort, and Asano]{kopiczko2024vera}
Dawid~Jan Kopiczko, Tijmen Blankevoort, and Yuki~M Asano.
\newblock Vera: Vector-based random matrix adaptation.
\newblock In \emph{The Twelfth International Conference on Learning Representations}, 2024{\natexlab{a}}.

\bibitem[Meng et~al.(2024)Meng, Wang, and Zhang]{meng2024pissa}
Fanxu Meng, Zhaohui Wang, and Muhan Zhang.
\newblock Pissa: Principal singular values and singular vectors adaptation of large language models.
\newblock \emph{Advances in Neural Information Processing Systems}, 37:\penalty0 121038--121072, 2024.

\bibitem[Zhang et~al.(2023)Zhang, Chen, Bukharin, He, Cheng, Chen, and Zhao]{zhang2023adaptive}
Qingru Zhang, Minshuo Chen, Alexander Bukharin, Pengcheng He, Yu~Cheng, Weizhu Chen, and Tuo Zhao.
\newblock Adaptive budget allocation for parameter-efficient fine-tuning.
\newblock In \emph{The Eleventh International Conference on Learning Representations}, 2023.

\bibitem[Lingam et~al.(2024)Lingam, Neerkaje, Vavre, Shetty, Gudur, Ghosh, Choi, Dimakis, Bojchevski, and sujay sanghavi]{lingam2024svft}
Vijay Lingam, Atula~Tejaswi Neerkaje, Aditya Vavre, Aneesh Shetty, Gautham~Krishna Gudur, Joydeep Ghosh, Eunsol Choi, Alex Dimakis, Aleksandar Bojchevski, and sujay sanghavi.
\newblock {SVFT}: Parameter-efficient fine-tuning with singular vectors.
\newblock In \emph{2nd Workshop on Advancing Neural Network Training: Computational Efficiency, Scalability, and Resource Optimization (WANT@ICML 2024)}, 2024.

\bibitem[Ba{\l}azy et~al.(2024)Ba{\l}azy, Banaei, Aberer, and Tabor]{balazy2024lora}
Klaudia Ba{\l}azy, Mohammadreza Banaei, Karl Aberer, and Jacek Tabor.
\newblock Lora-xs: Low-rank adaptation with extremely small number of parameters.
\newblock \emph{arXiv preprint arXiv:2405.17604}, 2024.

\bibitem[Trockman and Kolter(2021)]{trockman2021orthogonalizing}
Asher Trockman and J~Zico Kolter.
\newblock Orthogonalizing convolutional layers with the cayley transform.
\newblock In \emph{International Conference on Learning Representations}, 2021.

\bibitem[Wang and Manchester(2023)]{wang2023direct}
Ruigang Wang and Ian Manchester.
\newblock Direct parameterization of lipschitz-bounded deep networks.
\newblock In \emph{ICML}, 2023.

\bibitem[Jiang et~al.(2023)Jiang, Sablayrolles, Mensch, Bamford, Chaplot, de~las Casas, Bressand, Lengyel, Lample, Saulnier, Lavaud, Lachaux, Stock, Scao, Lavril, Wang, Lacroix, and Sayed]{jiang2023mistral7b}
Albert~Q. Jiang, Alexandre Sablayrolles, Arthur Mensch, Chris Bamford, Devendra~Singh Chaplot, Diego de~las Casas, Florian Bressand, Gianna Lengyel, Guillaume Lample, Lucile Saulnier, Lélio~Renard Lavaud, Marie-Anne Lachaux, Pierre Stock, Teven~Le Scao, Thibaut Lavril, Thomas Wang, Timothée Lacroix, and William~El Sayed.
\newblock Mistral 7b, 2023.
\newblock URL \url{https://arxiv.org/abs/2310.06825}.

\bibitem[Yu et~al.(2023)Yu, Jiang, Shi, Yu, Liu, Zhang, Kwok, Li, Weller, and Liu]{yu2023metamath}
Longhui Yu, Weisen Jiang, Han Shi, Jincheng Yu, Zhengying Liu, Yu~Zhang, James~T Kwok, Zhenguo Li, Adrian Weller, and Weiyang Liu.
\newblock Metamath: Bootstrap your own mathematical questions for large language models.
\newblock \emph{arXiv preprint arXiv:2309.12284}, 2023.

\bibitem[Cobbe et~al.(2021)Cobbe, Kosaraju, Bavarian, Chen, Jun, Kaiser, Plappert, Tworek, Hilton, Nakano, Hesse, and Schulman]{cobbe2021gsm8k}
Karl Cobbe, Vineet Kosaraju, Mohammad Bavarian, Mark Chen, Heewoo Jun, Lukasz Kaiser, Matthias Plappert, Jerry Tworek, Jacob Hilton, Reiichiro Nakano, Christopher Hesse, and John Schulman.
\newblock Training verifiers to solve math word problems.
\newblock \emph{arXiv preprint arXiv:2110.14168}, 2021.

\bibitem[Hendrycks et~al.(2021)Hendrycks, Burns, Kadavath, Arora, Basart, Tang, Song, and Steinhardt]{hendrycks2021measuring}
Dan Hendrycks, Collin Burns, Saurav Kadavath, Akul Arora, Steven Basart, Eric Tang, Dawn Song, and Jacob Steinhardt.
\newblock Measuring mathematical problem solving with the math dataset.
\newblock \emph{arXiv preprint arXiv:2103.03874}, 2021.

\bibitem[Zheng et~al.(2024)Zheng, Zhang, Shen, Liu, Lin, Fu, Chen, and Yue]{zheng2024opencodeinterpreter}
Tianyu Zheng, Ge~Zhang, Tianhao Shen, Xueling Liu, Bill~Yuchen Lin, Jie Fu, Wenhu Chen, and Xiang Yue.
\newblock Opencodeinterpreter: Integrating code generation with execution and refinement.
\newblock \emph{arXiv preprint arXiv:2402.14658}, 2024.

\bibitem[Chen et~al.(2021)Chen, Tworek, Jun, Yuan, de~Oliveira~Pinto, Kaplan, Edwards, Burda, Joseph, Brockman, Ray, Puri, Krueger, Petrov, Khlaaf, Sastry, Mishkin, Chan, Gray, Ryder, Pavlov, Power, Kaiser, Bavarian, Winter, Tillet, Such, Cummings, Plappert, Chantzis, Barnes, Herbert-Voss, Guss, Nichol, Paino, Tezak, Tang, Babuschkin, Balaji, Jain, Saunders, Hesse, Carr, Leike, Achiam, Misra, Morikawa, Radford, Knight, Brundage, Murati, Mayer, Welinder, McGrew, Amodei, McCandlish, Sutskever, and Zaremba]{chen2021evaluating}
Mark Chen, Jerry Tworek, Heewoo Jun, Qiming Yuan, Henrique~Ponde de~Oliveira~Pinto, Jared Kaplan, Harri Edwards, Yuri Burda, Nicholas Joseph, Greg Brockman, Alex Ray, Raul Puri, Gretchen Krueger, Michael Petrov, Heidy Khlaaf, Girish Sastry, Pamela Mishkin, Brooke Chan, Scott Gray, Nick Ryder, Mikhail Pavlov, Alethea Power, Lukasz Kaiser, Mohammad Bavarian, Clemens Winter, Philippe Tillet, Felipe~Petroski Such, Dave Cummings, Matthias Plappert, Fotios Chantzis, Elizabeth Barnes, Ariel Herbert-Voss, William~Hebgen Guss, Alex Nichol, Alex Paino, Nikolas Tezak, Jie Tang, Igor Babuschkin, Suchir Balaji, Shantanu Jain, William Saunders, Christopher Hesse, Andrew~N. Carr, Jan Leike, Josh Achiam, Vedant Misra, Evan Morikawa, Alec Radford, Matthew Knight, Miles Brundage, Mira Murati, Katie Mayer, Peter Welinder, Bob McGrew, Dario Amodei, Sam McCandlish, Ilya Sutskever, and Wojciech Zaremba.
\newblock Evaluating large language models trained on code, 2021.

\bibitem[Austin et~al.(2021)Austin, Odena, Nye, Bosma, Michalewski, Dohan, Jiang, Cai, Terry, Le, et~al.]{austin2021program}
Jacob Austin, Augustus Odena, Maxwell Nye, Maarten Bosma, Henryk Michalewski, David Dohan, Ellen Jiang, Carrie Cai, Michael Terry, Quoc Le, et~al.
\newblock Program synthesis with large language models.
\newblock \emph{arXiv preprint arXiv:2108.07732}, 2021.

\bibitem[Taori et~al.(2023)Taori, Gulrajani, Zhang, Dubois, Li, Guestrin, Liang, and Hashimoto]{alpaca}
Rohan Taori, Ishaan Gulrajani, Tianyi Zhang, Yann Dubois, Xuechen Li, Carlos Guestrin, Percy Liang, and Tatsunori~B. Hashimoto.
\newblock Stanford alpaca: An instruction-following llama model.
\newblock \url{https://github.com/tatsu-lab/stanford_alpaca}, 2023.

\bibitem[McCloskey and Cohen(1989)]{mccloskey1989catastrophic}
Michael McCloskey and Neal~J Cohen.
\newblock Catastrophic interference in connectionist networks: The sequential learning problem.
\newblock In \emph{Psychology of learning and motivation}, volume~24, pages 109--165. Elsevier, 1989.

\bibitem[French(1999)]{french1999catastrophic}
Robert~M French.
\newblock Catastrophic forgetting in connectionist networks.
\newblock \emph{Trends in cognitive sciences}, 3\penalty0 (4):\penalty0 128--135, 1999.

\bibitem[Wang et~al.(2024)Wang, Zhang, Su, and Zhu]{wang2024comprehensive}
Liyuan Wang, Xingxing Zhang, Hang Su, and Jun Zhu.
\newblock A comprehensive survey of continual learning: theory, method and application.
\newblock \emph{IEEE Transactions on Pattern Analysis and Machine Intelligence}, 2024.

\bibitem[Bafghi et~al.(2024)Bafghi, Harilal, Monteleoni, and Raissi]{bafghi2024parameter}
Reza~Akbarian Bafghi, Nidhin Harilal, Claire Monteleoni, and Maziar Raissi.
\newblock Parameter efficient fine-tuning of self-supervised vits without catastrophic forgetting.
\newblock In \emph{Proceedings of the IEEE/CVF Conference on Computer Vision and Pattern Recognition}, pages 3679--3684, 2024.

\bibitem[Dosovitskiy et~al.(2020)Dosovitskiy, Beyer, Kolesnikov, Weissenborn, Zhai, Unterthiner, Dehghani, Minderer, Heigold, Gelly, et~al.]{dosovitskiy2020image}
Alexey Dosovitskiy, Lucas Beyer, Alexander Kolesnikov, Dirk Weissenborn, Xiaohua Zhai, Thomas Unterthiner, Mostafa Dehghani, Matthias Minderer, Georg Heigold, Sylvain Gelly, et~al.
\newblock An image is worth 16x16 words: Transformers for image recognition at scale.
\newblock In \emph{International Conference on Learning Representations}, 2020.

\bibitem[Deng et~al.(2009)Deng, Dong, Socher, Li, Li, and Fei-Fei]{deng2009imagenet}
Jia Deng, Wei Dong, Richard Socher, Li-Jia Li, Kai Li, and Li~Fei-Fei.
\newblock Imagenet: A large-scale hierarchical image database.
\newblock In \emph{2009 IEEE conference on computer vision and pattern recognition}, pages 248--255. IEEE, 2009.

\bibitem[Kopiczko et~al.(2024{\natexlab{b}})Kopiczko, Blankevoort, and Asano]{kopiczko2024elora}
Dawid~Jan Kopiczko, Tijmen Blankevoort, and Yuki~M Asano.
\newblock Elora: Efficient low-rank adaptation with random matrices.
\newblock In \emph{The Twelfth International Conference on Learning Representations}, 2024{\natexlab{b}}.

\bibitem[Netzer et~al.(2011)Netzer, Wang, Coates, Bissacco, Wu, and Ng]{Netzer2011}
Yuval Netzer, Tao Wang, Adam Coates, Alessandro Bissacco, Bo~Wu, and Andrew~Y Ng.
\newblock Reading digits in natural images with unsupervised feature learning.
\newblock In \emph{Advances in Neural Information Processing Systems (NIPS)}, 2011.

\bibitem[Krizhevsky et~al.(2009)Krizhevsky, Hinton, et~al.]{krizhevsky2009learning}
Alex Krizhevsky, Geoffrey Hinton, et~al.
\newblock Learning multiple layers of features from tiny images.
\newblock 2009.

\bibitem[Bossard et~al.(2014)Bossard, Guillaumin, and Van~Gool]{bossard2014food}
Lukas Bossard, Matthieu Guillaumin, and Luc Van~Gool.
\newblock Food-101--mining discriminative components with random forests.
\newblock In \emph{Computer vision--ECCV 2014: 13th European conference, zurich, Switzerland, September 6-12, 2014, proceedings, part VI 13}, pages 446--461. Springer, 2014.

\bibitem[Bhatia and Kittaneh(1990)]{bhatia1990singular}
Rajendra Bhatia and Fuad Kittaneh.
\newblock On the singular values of a product of operators.
\newblock \emph{SIAM Journal on Matrix Analysis and Applications}, 11\penalty0 (2):\penalty0 272--277, 1990.

\bibitem[Bhatia(2013)]{bhatia2013matrix}
Rajendra Bhatia.
\newblock \emph{Matrix analysis}, volume 169.
\newblock Springer Science \& Business Media, 2013.

\bibitem[Loshchilov and Hutter(2019)]{loshchilov2018decoupled}
Ilya Loshchilov and Frank Hutter.
\newblock Decoupled weight decay regularization.
\newblock In \emph{International Conference on Learning Representations}, 2019.

\end{thebibliography}
